\documentclass[sigconf]{acmart}

\AtBeginDocument{%
	}

\setcopyright{acmcopyright}
\copyrightyear{2018}
\acmYear{2018}
\acmDOI{XXXXXXX.XXXXXXX}

\acmConference[Conference acronym 'XX]{Make sure to enter the correct
	conference title from your rights confirmation emai}{June 03--05,
	2018}{Woodstock, NY}
\acmPrice{15.00}
\acmISBN{978-1-4503-XXXX-X/18/06}

\newcommand{\Ber}{\mathrm{Ber}}

\newenvironment{lemmaprime}[1]
{%
	\renewcommand{\thetheorem}{\ref{#1}$'$}%
	\addtocounter{theorem}{-1}%
	\begin{lemma}}
	{\end{lemma}}

\usepackage{balance}

\DeclareMathOperator{\bE}{{\mathop{\mathbf{E}}}}

\newcommand{\E}{\mathcal{E}}

\newcommand{\abs}[1]{\left|{#1}\right|}

\newcommand{\A}{\mathcal{A}}

\newcommand{\W}{\mathcal{H}}

\newcommand{\D}{\mathcal{D}}

\newtheorem{theorem}{Theorem}
\newtheorem{lemma}[theorem]{Lemma}
\newtheorem{fact}[theorem]{Fact}
\newtheorem{claim}[theorem]{Claim}
 \newtheorem{remark}[theorem]{Remark}

 \usepackage[ruled, vlined]{algorithm2e}

\newcommand{\fullversion}{}

\author{Nikolai Karpov}
\affiliation{%
	\institution{Indiana University}
	\city{Bloomington} 
	\state{IN} 
	\postcode{47408}
	\country{USA}
}
\email{nkarpov@iu.edu}

\author{Qin Zhang}
\affiliation{%
	\institution{Indiana University}
	\city{Bloomington} 
	\state{IN} 
	\postcode{47408}
	\country{USA}
}
\email{qzhangcs@iu.edu}
\authornote{Corresponding author.}

\ifdefined\fullversion
\else
\copyrightyear{2024}
\acmYear{2024}
\setcopyright{acmlicensed}\acmConference[SPAA '24]{Proceedings of the 36th ACM Symposium on Parallelism in Algorithms and Architectures}{June 17--21, 2024}{Nantes, France}
\acmBooktitle{Proceedings of the 36th ACM Symposium on Parallelism in Algorithms and Architectures (SPAA '24), June 17--21, 2024, Nantes, France}
\acmDOI{10.1145/3626183.3659957}
\acmISBN{979-8-4007-0416-1/24/06}
\fi

\begin{document}

\title{Parallel Best Arm Identification in Heterogeneous Environments}
\titlenote{Authors are supported in part by NSF CCF-1844234 and CCF-2006591.}


\begin{abstract}
	In this paper, we study the tradeoffs between the {\em time} and the {\em number of communication rounds} of the best arm identification problem in the heterogeneous collaborative learning model, where multiple agents interact with possibly different environments and they want to learn in parallel an objective function in the aggregated environment.  By proving almost tight upper and lower bounds, we show that collaborative learning in the heterogeneous setting is inherently more difficult than that in the homogeneous setting in terms of the time-round tradeoff.
\end{abstract}



\ifdefined\fullversion
\else

\fi


\received{20 February 2007}
\received[revised]{12 March 2009}
\received[accepted]{5 June 2009}

\maketitle

\section{Introduction}
\label{sec:intro}

As data continue to grow, multi-agent learning has emerged as an important direction in scalable machine learning and has attracted much attention under the name of {\em federated learning}~\citep{KMR15,KMRR16,MMR+17}, where multiple agents try to learn an objective function in parallel via communication.  While the majority of work in federated learning focuses on the distributed training of neural networks, a few papers \citep{TZZ19,KZZ20,WHCW20} studied parallel reinforcement learning problems in a very similar model named the {\em collaborative learning}  (CL) model. However, most work in the literature of collaborative learning only considered the {\em homogeneous} environment (or, IID data), in which agents interact with the same data distribution. Real world applications of multi-agent learning often involve {\em heterogeneous} environments (or, non-IID data), in which agents interact with possibly different data distributions.\footnote{We felt that the words ``IID/non-IID", which are  widely used in the literature of federated learning, are somewhat confusing. In the rest of this paper, we will use the words ``homogeneous" and ``heterogeneous" to denote the scenarios where agents interact with identical and different data distributions, respectively.} Indeed, heterogeneous environments have been identified as a key feature of the federated learning model \citep{corr/abs-1912-04977}.

In this paper, we investigate heterogeneous collaborative learning. We will use a basic problem in bandit theory named {\em best arm identification in multi-armed bandits} (BAI) as a vehicle to deliver the following message: {\em  Collaborative learning in the heterogeneous environment is provably more difficult than that in the homogeneous environment w.r.t.\ communication rounds.}

In the following, we first introduce the BAI problem and the CL model, and then summarize our results and contributions.  We conclude the section by discussing related work.

\vspace{2mm}
\noindent{\bf Best Arm Identification in Multi-Armed Bandits.\ \ }   In BAI, we have $n$ arms, each of which is associated with an unknown distribution $\D_i\ (i \in [n])$ with support $[0, 1]$.\footnote{We use $[n]$ to denote $\{1, 2, \ldots, n\}$.}
We aim at identifying the arm whose associated distribution has the largest mean by a sequence of $T$ pulls. In each arm pull, we choose an arm based on the previous pulls and outcomes, and obtain a sample from the arm's associated distribution.  Assuming that each pull takes unit time, we call $T$ the {\em time horizon}. The goal of BAI is to identify the arm with the highest mean with the smallest error probability under time horizon $T$.\footnote{For readers who are familiar with the bandit literature, we are considering the {\em fixed-time/budget} best arm identification. Another version of this problem is called  {\em fixed-confidence}, where we want to solve BAI with a fixed error probability $\delta$ with the smallest number of pulls.}

BAI is a basic problem in bandit theory and reinforcement learning, and has been studied extensively in the literature since 1950s (e.g., \citep{Bechhofer58,Paulson64,EMM02,EMM06,ABM10,KKS13,JMNB14,CLQ17}).   The problem has numerous real-world applications, including clinical trials, article/ad/channel selection, computer game play, financial portfolio design, adaptive routing, crowd-sourced ranking, hyperparameter optimization, etc.

Let $I  = \{1, 2, \ldots, n\}$ be an input instance of $n$ arms.
W.l.o.g., we assume that there is a unique best arm, which is denoted by $i_*$.  Let $\mu_*$ be the mean of $\D_{i_*}$, and for any $i \in [n]$, let $\mu_i$ be the mean of $\D_i$.   Let $\Delta_i = \mu_* - \mu_i$ be the {\em mean gap} of the best arm and the $i$-th  arm.
The {\em instance complexity} of BAI on input  $I$ is defined as:
\begin{equation}
	\label{eq:instance-complexity}
	H(I) \triangleq \sum_{i \in [n], i \neq i_*} {1}/{\Delta_i^2}.
\end{equation}
Intuitively, the term $1/\Delta_i^2$ is the number of pulls needed to separate the best arm and the $i$-th  arm with a good probability.   We will sometimes write $H \triangleq H(I)$ for convenience. It is known that under time horizon $\tilde{O}(H)$, there exists a centralized algorithm that solves BAI with probability $0.99$ \citep{ABM10}.\footnote{For the convenience of presentation, we sometimes use `$\tilde{\ \ }$' on $O, \Omega, \Theta$ to hide non-critical logarithmic factors. 
	All logarithmic factors will be spelled out in our theorems explicitly.
} On the other hand, no centralized algorithm can solve the BAI problem with probability $0.99$ under time horizon $H$ \citep{CL16}. 

\vspace{2mm}
\noindent{\bf The Collaborative Learning Model.\ \ }
Most study for BAI has been done in the centralized model, in which just one agent pulls the set of arms sequentially.  \cite{TZZ19,KZZ20} studied BAI in the collaborative learning (CL) model, where there are $K$ agents, who try to learn the best arm in parallel via communication.   The learning proceeds in rounds. 
In each round, each agent takes a sequence of pulls (one at each time step) and observes the outcomes.  At the end of each round there is a communication phase; the agents communicate with each other to exchange newly observed information and determine the number of time steps for the next round (the length of the first round is determined at the beginning of the first round). At the end of the last round, all agents have to output the same answer without any further communication.  The goal of BAI in the CL model is for all agents to output the correct answer with the smallest error probability under time horizon $T$ (i.e., the number of time steps over all rounds) and the number of rounds $R$. Note that the number of communication phases is $(R - 1)$, since we do not allow any communication at the end of the last round.

Depending on whether the agents have real-time computing and policy-updating ability, the CL algorithms are divided into two categories: {\em adaptive} and {\em non-adaptive}. In the adaptive case, agents can change their pull policies at each time step based on new observations. While in the non-adaptive case, policy updates can only happen at the beginning of each round. 
In this paper, we focus on the adaptive case; the lower bound proof for the adaptive case is more challenging than the non-adaptive case due to agents' local adaptivity within a round.

Minimizing communication in the CL model is critical due to network bandwidth constraints and latency, energy consumption (think of deep-sea/outer-space exploration), and data usage (e.g., if messages are sent by mobile devices). In this paper, we mainly focus on  the round complexity. Like parallel/distributed computation models such as MapReduce, initiating a new round of learning process can be very expensive due to various communication overheads.  
The communication cost (i.e., the total number of bits exchanged between agents) of our algorithm is optimal up to a logarithmic factor based on a recent lower bound result in \citep{KZ23} (see Remark~\ref{rem:ub-cc}).

\vspace{2mm}
\noindent{\bf Heterogeneous Environments.\ \ }
In the CL model studied by \cite{TZZ19} and \cite{KZZ20}, each agent interacts with the same environment; for the BAI problem in particular, by pulling the same arm, the agents sample from the same data distribution. However, as mentioned earlier, heterogeneous environments are inherent in many real-world collaborative learning applications.  

For example, in the setting of channel selection in cognitive radio networks, a base station utilizes a number of mobile devices (e.g., cell phones) to select the best channel for data transformation in a particular area. Here each mobile device represents an agent and each channel represents an arm.  At each time step, an agent selects a channel and attempts to transmit a message. If the message is successfully delivered, the agent receives a reward of $1$; otherwise, the reward is $0$. This corresponds to the bandit setting.  Since mobile devices sit at different geographic locations, the channel availability distributions they observe may be very different.   The base station needs to identify the best arm with respect to the {\em aggregation} of local channel availability distributions.  
Another example is the task of item-selection in recommendation systems, where a group of servers work together to learn the globally most popular item via communication, while each server can only interact with users in a certain region (and thus get samples from a distinct data distribution).

In BAI with heterogeneous environments, by pulling the same arm, the agents sample from possibly different distributions.  Let $\pi_{i,k}$ be the distribution associated with the $i$-th arm that the $k$-th agent samples from, and let $\mu_{i,k}$ be the mean of $\pi_{i,k}$. Define the global mean of the $i$-th arm as
\begin{equation}
	\label{eq:unweighted}
	\mu_i \triangleq \frac{1}{K} \sum_{k \in [K]}\mu_{i,k}.
\end{equation}
Our task is to identify the arm $i_*$ with the largest global mean, while each agent $k \in [K]$ can only pull each arm $i \in [n]$ under its local distribution $\pi_{i,k}$.  In the heterogeneous setting, we define mean gaps based on the global means, that is, $\Delta_i = \mu_* - \mu_i$, and the instance complexity $H$ again as $\sum_{i \in [n], i \neq i_*} {1}/{\Delta_i^2}$.

\vspace{2mm}
\noindent{\bf Our Results.\ \ }
The main result of this paper is the following impossibility result.  

\begin{theorem}[Main Theorem]
	\label{thm:lb-main}
	For any $1 \le R \le \frac{\log n}{24\log\log n}$ and any $T < H n^{\Omega\left(\frac{1}{R}\right)}/K$, any $R$-round $T$-time $K$-agent algorithm that solves $n$-arm BAI  in the heterogeneous CL model  has a success probability less than $0.99$.
\end{theorem}

We complement the impossibility result by the following algorithmic result.  
\begin{theorem}
	\label{thm:ub-main}
	For any $R \ge 1$ and any $T \ge c_T H n^{\frac{1}{R}} / K$ for a universal constant $c_T$, there exists a $R$-round $T$-time $K$-agent algorithm that solves $n$-arm BAI in the heterogeneous CL model with probability $0.99$.
\end{theorem}

We note that  for a fixed time budget, the number of rounds $R$ in the lower and upper bounds in Theorem~\ref{thm:lb-main} and Theorem~\ref{thm:ub-main} match up to a constant factor.  

We would like to highlight a couple of points regarding Theorem~\ref{thm:lb-main}. First, this is the first lower bound result that addresses the {\em local agent adaptivity} in the CL models.  In particular, it shows that the capacity of each agent to utilize newly observed information within each round does {\em not} contribute to reducing the round complexity in the heterogeneous CL model.  This is in stark contrast with the homogeneous CL model in which local agent adaptivity can significantly reduce the round complexity.  Second, our hard input distribution for proving Theorem~\ref{thm:lb-main} is the first one that uses {\em asymmetric} arm means constructions. It exploits the heterogeneous property, enabling us to establish a higher lower bound for BAI than the one presented in the homogeneous setting \citep{TZZ19}. We will give a more detailed technical overview in Section~\ref{sec:overview}.

\subsection{Related Work}
\label{sec:related}

We summarize previous work that is closely related to this paper in the CL model, and refer readers to the book by \cite{LS20} for an overview on BAI in the centralized model.

The (homogeneous) CL model was first used in the work \cite{HKK+13} for studying multi-agent BAI, but the model was not formally defined there. The results for fixed-time BAI in \cite{HKK+13} only consider the special case where there is only one communication phase (i.e., $R = 2$).  The CL model was rigorously formulated in \cite{TZZ19}, where the authors obtained almost tight tradeoffs between the learning time and the round cost for BAI.  The followup work \citep{KZZ20} extended this line of research to the top-$m$ arm identifications problem.  \cite{WHCW20} studied regret minimization in multi-armed bandits in essentially the same model, but it focused on the total bits of communication exchanged between the agents (or, the communication cost) instead of the number of rounds. Recently, \cite{KZ23} studied the tradeoff between the learning time and the communication cost in the CL model for BAI, and  \cite{ALGY22} studied linear bandits in a similar setting.

The authors of \cite{RVK22} studied BAI and regret minimization in multi-armed bandits in a model similar to the CL model, but mainly in the fixed-confidence setting.  That is, their algorithm takes a confidence parameter $\delta$ (instead of a time horizon $T$) as an input, and try to use the smallest possible number of time steps to identify the best arm with probability $(1 - \delta)$. Their lower bound results are proved for the setting that agents can communicate at each time step.

In the heterogeneous CL model, \cite{SS21,SSY21} studied regret minimization in multi-armed bandits.  The authors considered the communication cost  of the CL algorithms, but the cost has been embedded into the regret formulation.  \cite{MHP21} studied BAI in the CL model where arms are partitioned into groups, and each agent can only pull arms from one particular group. This model can be thought as a special case of the heterogeneous CL model studied in this paper, where for any arm $i \in [n]$, there exists a unique agent $k \in [K]$ such that $\mu_{i,k} > 0$, while $\mu_{i,k'} = 0$ for all $k' \in [K] \backslash \{k\}$. This special case does {\em not} capture the inherent difficulty of the heterogeneous CL model where the information about a particular arm can spread over multiple agents, and their results cannot be generalized to the heterogeneous CL  model.

\section{Technical Overview of The Main Result}
\label{sec:overview}

Before delving into the full proof of our main result (Theorem~\ref{thm:lb-main}), which is very technical, we would like to provide an overview.

We note that all the parameters used in this technical overview are merely for the illustration purpose. They may {\em not} correspond to the actual, typically more complex, parameters used in the actual proof.  We will also frequently {\em ignore lower-order logarithmic terms} for the sake of readability.

\vspace{2mm}
\noindent{\bf Generalized Round Elimination and Challenges.\ \ }
Let us start by briefly illustrating the generalized round elimination technique introduced in \cite{TZZ19}, and then explain the challenges in applying it in the heterogeneous setting.  

Generalized round elimination can be thought as an induction on a sequence  of hard distribution classes $\D_0, \D_1, \ldots, \D_R$, where $\D_0 = \{\phi\}$ consists of the original hard input distribution $\phi$.\footnote{We need to use input {\em distribution} instead of a single hard input instance because we are proving a lower bound for randomized algorithms.  By Yao's minimax lemma~\citep{Yao77}, we can instead prove a lower bound for deterministic algorithms on a hard input distribution.}  At the $i$-th induction step, we show that for any input distribution in $\sigma \in \D_{i-1}$, if the agents do not conduct enough {\em non-adaptive} pulls (due to the time budget constraint) in a round, then after some ``input massage" which will only make the problem easier, the posterior distribution $\sigma'$ belongs to $\D_i$.  For the base case, we show that {\em no} $0$-round CL algorithm can solve the problem for any distribution $\sigma \in \D_R$ with a non-trivial success probability.  We can thus prove that {\em no} $R$-round algorithm for solving the problem on the original input distribution $\phi$ with  a non-trivial success probability.

The lower bound proof using generalized round elimination in \cite{TZZ19} was carried out on non-adaptive algorithms in the homogeneous CL setting.  For adaptive algorithms, only in the case when $n \le K$ (that is, the number of arms is no more than the number of agents), we can show that adaptive pulls do {\em not} have much advantage against non-adaptive pulls via a coupling argument.  This is why in \cite{TZZ19}, only a  $\Omega\left(\frac{\log \min\{K, n\}}{\log\log \min\{K, n\}}\right)$ round lower bound can be proved for adaptive algorithms.  As mentioned, for the heterogeneous CL setting, our goal is to prove an ${\Omega}\left(\frac{\log n}{\log\log n}\right)$ round lower bound for adaptive algorithms for any value $n$. To this end, we must design a new, harder input distribution that leverages the heterogeneous property of the data distributions.

\vspace{2mm}
\noindent{\bf An Interleaved Local Mean Construction.\ \ }
Our new input distribution for heterogeneous data is easier to visualize with two agents, Alice and Bob, but it can easily be extended to multiple agents.

We will focus on the ${\Omega}(\log n)$ round case (ignoring a $\log\log$ factor), while our lower bound result covers the entire time-round tradeoff.    The formal definition of our hard input distribution and its properties can be found in Section~\ref{sec:hard-distribution}.

Our hard input distribution has $L = \Theta(\log n)$ terms, with odd terms held by Alice and even terms held by Bob. The global mean of each arm can be written as
\begin{equation*}
	\label{eq:global-mean}
	\mu = \frac{1}{2} + \sum_{\ell=1}^L \frac{X_\ell}{4^\ell},
\end{equation*}
where $X_1, \ldots, X_L \in \{0, 1\}$ are i.i.d.\ Bernoulli random variables with mean $\frac{1}{2}$.  When $X_1 = \cdots = X_L = 1$, $\mu$ achieves its maximum possible value. The local mean of each arm at Alice's side is
$$\mu^A = \frac{1}{2} + \sum_{\ell:1 \le 2\ell+1 \le L} \frac{2X_{2\ell+1}}{4^{2\ell+1}},$$
and that at Bob's side is
$$\mu^B = \frac{1}{2} +  \sum_{\ell:1 \le 2\ell \le L} \frac{2X_{2\ell}}{4^{2\ell}}.$$
Note that $\mu = (\mu^A + \mu^B)/2$.  Let $\pi$, $\pi^A$ and $\pi^B$ be the underlying distributions of $\mu$, $\mu^A$, and $\mu^B$.

\vspace{2mm}
\noindent{\bf Proof Intuition and New Challenges.\ \ }
We say an arm is at level $\ell$ if $X_1 = \ldots = X_\ell = 1$ and $X_{\ell+1} = 0$ (if $\ell < L$).
The high-level intuition of proving an $\Omega(L) (= \Omega(\log n))$ lower bound is that Alice and Bob must learn the set of $n$ arms level by level under a time budget $\tilde{O}(H)$, where $H$ is the input instance complexity.  That is, at the end of the $\ell$-th round, they can only identify and eliminate those arms that are in the first $\ell$ levels, while for the remaining arms the uncertainty is still large.  As a result, they need $L = \Omega(\log n)$ rounds to identify the best arm.  Ideally, we hope to show that at each odd round $\ell$, Alice is able to identify and eliminate those arms who are in level $\ell$ but not higher, while Bob is {\em not} able to do much as he lacks information about $X_{\ell}$ of each arm. And a similar situation holds at each even round $\ell$ with Alice and Bob's positions swapped.

The difficulty in formalizing the above  intuition is that it is actually {\em possible} for each party to learn information about the bits (i.e., the $X_i$'s) at {\em all levels} using their local samples and messages received from the other party. What we need to show is that this information is {\em not} enough to allow parties to ``jump" $\omega(1)$ levels after each round given the total sample budget.

We try to formalize this intuition using generalized round elimination.  There are two challenges in proving a $\Omega(\log n)$ round lower bound for BAI in the heterogeneous CL model. 
\begin{enumerate}
	\item Explicit forms of distribution classes like those used in \cite{TZZ19} in the homogeneous setting are difficult to obtain  in the heterogeneous setting due to the intricate structures of the hard input distributions $\mu^A$ and $\mu^B$.  
	
	\item Since the coupling argument which reduces adaptive CL algorithms to non-adaptive CL algorithms is inapplicable when $n > K$, we have to prove the lower bound for adaptive CL algorithms directly.
\end{enumerate}

In the following, we briefly illustrate how we address these two challenges.

\vspace{2mm}
\noindent{\bf Implicit Forms of Distribution Classes.\ \ }
Our first technical innovation is that we implicitly define the classes of distributions for the generalized round elimination by quantifying the relationship between each distribution in the class and the original hard input distribution.   The discussion below is again a simplified version of the actual construction, whose details can be found in Section~\ref{sec:class-distribution}.

The distribution classes for Alice and Bob are defined in a similar way.  Here we use Alice for example, and define the distribution classes $\D_\ell^A\ (\ell = 0, 1, \ldots)$ for Alice.  The combined distribution class will be denoted by $\D_\ell = (\D_\ell^A, \D_\ell^B)$, where $\D_\ell^B$ is the one defined for Bob.   

Let $S_\ell^A$ be the set of all possible local means at Alice's side for arms in levels $\ell, \ldots, L$, and let $\varsigma = n^{1/L}$.  For each level $\ell = 0, 1, \ldots, L$, define input class $\D_\ell^A$ to be the set of distributions $\sigma^A$ with support $S_\ell^A$ such that
\begin{equation}
	\label{eq:ratio}
	\forall x, y \in S_\ell^A: \frac{\Pr_{\mu^A \sim \sigma^A}[\mu^A = x]}{\Pr_{\mu^A \sim \sigma^A}[\mu^A = y]} = \frac{\Pr_{\mu^A \sim \pi^A}[\mu^A = x]}{\Pr_{\mu^A \sim \pi^A}[\mu^A = y]} \cdot e^{\pm \frac{\ell}{\varsigma}}.
\end{equation}
Note that $\D_0^A = \{\pi^A\}$ where $\pi^A$ is the original input distribution at Alice's side.  Intuitively, Equation~(\ref{eq:ratio}) states that for any distribution $\sigma^A \in \D_\ell^A$, the ratio between the probability mass on any two possible mean values in $\sigma^A$ is close to that in the original input distribution $\pi^A$.  Consequently, if the original input distribution $\pi^A$ is quite ``uncertain", then any distribution $\sigma^A \in \D_\ell^A$ is also quite uncertain.  The extra $e^{\pm \frac{\ell}{\varsigma}}$ is a relaxation term that counts the influence of the pull outcomes in the first $\ell$ rounds on the posterior distribution of $\pi^A$.

We have the following lemma.
Its formal statement can be found in Lemma~\ref{lem:distribution-class} in Section~\ref{sec:class-distribution}.  We slightly rewrite and simplify the statement here for the illustration purpose.
\begin{lemma}
	\label{lem:distribution-class-simplied}
	For any $\ell \in \{0, 1, \ldots, L-1\}$, any distribution $\sigma^A \in \D_\ell^A$, and any good sequence of pull outcomes $\theta = (\theta_1, \ldots, \theta_q)$ in the current round, the posterior distribution of $\sigma^A$ after observing a sequence of pull outcomes being $\theta$ and conditioning on the mean of the arm $\mu^A \in S_{\ell+1}^A$, denoted by $(\sigma^A\ |\ \theta, \mu^A \in S_{\ell+1}^A)$, belongs to the distribution class $\D_{\ell+1}^A$.
\end{lemma}

On the other hand, we can also show that the pull sequence $\theta$ is good with high probability if its length is not too large, which holds if there is a  time budget constraint.  

Lemma~\ref{lem:distribution-class-simplied} helps in establishing the foundation of the induction in the round elimination without having to go through the spelling of the posterior distributions after a round of arm pulls.

\vspace{2mm}
\noindent{\bf A Lower Bound  for Adaptive CL Algorithms.\ \ }  Our second technical contribution is to prove the lower bound for adaptive CL algorithms directly, instead of via a reduction from a lower bound for non-adaptive CL algorithms.  The details can be found in Section~\ref{sec:lb-two}.

Let us first recall the proof for non-adaptive algorithms in \cite{TZZ19}. After the first round of pulls, we set a threshold $\eta$ and {\em publish} those arms who have been pulled more than $\eta$ times in the first round; we call these arms the {\em heavy arms}.  By publishing an arm we mean revealing its mean to all agents; note that this will only make the problem easier, and consequently make the lower bound proof stronger.  This arm publishing procedure is what we formerly referred to as the `` input massage".

We use the arm publishing procedure to ensure that the means of remaining arms belong to the next class (i.e., $\mu^A \in S_{\ell+1}^A$).  For the set of distribution classes $\D_\ell^A\ (\ell = 0, 1, \ldots)$  used in this paper,  we can use Lemma~\ref{lem:distribution-class-simplied} to show that the posterior distribution of some $\sigma^A \in \D_\ell^A$, after the publishing procedure, belongs to the next distribution class $\D_{\ell+1}^A$. 

In order for the induction to proceed, we need to make sure that if we publish all heavy arms, the probability of the best arm being published is small, since otherwise the problem would already be solved and the round elimination process {\em cannot} continue.  This is easy to do with non-adaptive algorithms, because the whole pull sequence and consequently the set of heavy arms are determined at the beginning of each round.  If the time budget is small, then the number of heavy arms must be small. Consequently, the probability that the set of heavy arms contain the best arm is also small, because all arms are almost equally uncertain at the beginning of each round. In other words, the set of heavy arms would be an almost {\em random subset} of all arms.  

Adaptive algorithms, however, can utilize their adaptivity to look for arms with high means and make more pulls on those arms.  

To handle this challenge, we choose to explicitly analyze for each heavy arm its probability of being the best arm after the first round of pulls, and then show that the sum of these probabilities is small.  This analysis is much more complicated than that for the non-adaptive algorithms. We try to illustrate the main ideas below.

The key to the analysis for each individual arm is that, because of the interleaved mean structure, Alice misses most information of half of the terms held by Bob. Without this information, her adaptivity cannot help much in the task of identifying which arm is more likely to be the best arm.  On the other hand, the time budget constraint also prevents Alice from extracting and revealing to Bob too much information about her local means of arms which are {\em not} published  in the next round (see the algorithm {\em Arm Publishing and Additional Pulls} in Section~\ref{sec:induction} for details on how we publish arms). A similar argument holds for Bob. Despite appearing natural, it is highly non-trivial to put this intuition into a formal proof since we need to carefully bound the ``help" of the historical information exchange. The adaptivity of the algorithm further complicates the description of the posterior distribution of the arms after one round of pulls.  Fortunately, our implicit representation of the distribution classes is flexible enough to handle this additional complexity.

Finally, we would like to mention that due to technical needs, in each step of our induction we have to ``consume'' multiple, but still $O(1)$, levels out of the $L$ levels of arms, but this will {\em not} change the asymptotic round bound.

\vspace{2mm}
\noindent{\bf Generalizing to $K$ Parties.\ \ }  Finally, we comment that we can easily generalize the lower bound for $2$ agents to $K$ agents via a reduction.  See Section~\ref{sec:general-K} for details.

\section{The Impossibility Result}
\label{sec:lb}

In this section, we give the proof to Theorem~\ref{thm:lb-main}.

We start with the case when there are two agents (i.e., $K = 2$), and then generalize the results to all $K$.  
Below are a few notations that we will be using in this section.
\begin{itemize}
	\item $R$: The number of rounds used by the algorithm.  We will focus on the range $1 \le R \le \frac{\log n}{24 \log\log n}$.
	
	\item $L \triangleq 6R$: The number of terms in the means of arms in the hard input distribution. 
	
	\item $\eta \triangleq n^{\frac{1}{2L}} = n^{\frac{1}{12R}}$: Intuitively, it is the ratio between the maximum contributions of consecutive terms in the mean construction.  For $1 \le R \le \frac{\log n}{24 \log\log n}$, we always have $\eta \ge \log^2 n$.
	
	\item  $\zeta \triangleq \frac{\sqrt{\eta}}{2^7} = \frac{n^{\frac{1}{24R}}}{2^7}$: A parameter related to the time of the CL algorithm.
	
	\item $\gamma \triangleq \frac{\eta}{2^7} = \frac{n^{\frac{1}{12R}}}{2^7} = \Theta(\zeta^2)$: A parameter for the convenience of the presentation.
	
	\item $\Ber(\mu)$ denotes the Bernoulli distribution with mean $\mu$.
\end{itemize}

For convenience, when we write $c = a \pm b$ (or $c \pm d = a \pm b$), we mean $c \in [a-b, a+b]$ (or  $[c-d, c+d] \subseteq [a-b, a+b]$). Without this simplification, some formulas may be difficult to read.

We will use the following standard concentration bound.
\begin{lemma}[Chernoff-Hoeffding Inequality]
	\label{lem:chernoff}
	Let $X_1, \dots, X_n \in [a_i, b_i]$ be independent random variables. Let ${X} = \sum_{i=1}^{n} X_i$. For any $t \geq 0$, it holds that
	\begin{eqnarray*}
		&& \Pr\left[{X} \ge \bE[{X}] + t\right] \leq  \exp\left(-\frac{2t^2}{\sum_{i = 1}^n (b_i - a_i)^2} \right), \quad \text{and}  \\
		&& \Pr\left[{X} \le \bE[{X}] - t\right] \leq  \exp\left(-\frac{2t^2}{\sum_{i = 1}^n (b_i - a_i)^2} \right).
	\end{eqnarray*}
\end{lemma}

\vspace{2mm}

In the rest of this section, we first introduce the hard input distribution that we use to prove the lower bound and discuss its properties.  We then introduce the classes of distributions on which we will perform the generalized round elimination.   After these preparation steps, we present our main lower bound proof for $K = 2$, and then extend it to the general case.

\subsection{The Hard Input Distribution (When $K = 2$) and Its Properties.}  
\label{sec:hard-distribution}

Define random variable
\begin{equation}
	\label{eq:a-1}
	\mu = \mu(X_1, \ldots, X_L) = \frac{1}{2} + \sum_{\ell=1}^L \frac{X_\ell}{\eta^\ell},
\end{equation} 
where for each $\ell \in [L]$, $\ X_\ell \sim \Ber\left(\eta^{-2}\right)$ are drawn independently.   Let $\pi$ be the distribution of random variable $\mu$.

Let $(\mu_1, \ldots, \mu_n) \sim \pi^{\otimes n}$, where $\mu_i$ is the global mean of arm $i$.  We divide each $\mu_i$ into two local means $\mu_i^A$ and $\mu_i^B$ for Alice and Bob respectively, where
\begin{equation*}
	\label{eq:a-2}
	\mu^A = \frac{1}{2} +  \sum_{\ell : 1 \le 2\ell+1 \le L} \frac{2X_{2\ell+1}}{\eta^{2\ell+1}}, \quad \text{and} \quad \mu^B = \frac{1}{2} +  \sum_{\ell : 1 \le 2\ell \le L} \frac{2X_{2\ell}}{\eta^{2\ell}}.
\end{equation*} 
That is, Alice takes all odd terms in the summation of (\ref{eq:a-1}), and Bob takes all even terms in the summation of (\ref{eq:a-1}); the factor $2$ is just to make sure that $\mu = (\mu^A + \mu^B)/2$. It is clear that $\mu^A$ and $\mu^B$ are independent, because they depend on disjoint subsets of $\{X_1, \ldots, X_L\}$.  Let $\pi^A$ and $\pi^B$ be the underlying distributions of random variables $\mu^A$ and $\mu^B$, respectively.  We can write $\pi = (\pi^A, \pi^B)$.

\vspace{2mm}
\noindent{\bf Key Properties of the Support of Distribution $\pi = (\pi^A, \pi^B)$.\ \ }
For each $\ell \in \{0, 1, \ldots, L\}$, we define the following two sets:
\begin{equation}
	\label{eq:d-1}
	S^{A}_\ell \triangleq \left\{\left.\frac{1}{2} + \sum_{k : 1 \le 2k+1 \le \ell} \frac{2}{\eta^{2k+1}} + \sum_{k : \ell < 2k+1 \le L} \frac{2X_{2k+1}}{\eta^{2k+1}}\ \right|\  X_{2k+1} \in \{0, 1\}\right\},
\end{equation}
and 
\begin{equation}
	\label{eq:d-2}
	S^{B}_\ell \triangleq \left\{\left.\frac{1}{2} +  \sum_{k : 1 \le 2k \le \ell} \frac{2}{\eta^{2k}} + \sum_{k : \ell < 2k \le L} \frac{2X_{2k}}{\eta^{2k}}\ \right|\  X_{2k} \in \{0, 1\}\right\}.
\end{equation}
Intuitively, the set $S_\ell^A$ consists of values in $\textrm{supp}(\pi^A)$ with $X_1 = X_3 = \ldots = X_{\ell'} = 1$, where $\ell'$ is the largest odd integer no more than $\ell$. And the set $S_\ell^B$ consists of values in $\textrm{supp}(\pi^B)$ with $X_2 = X_4 = \ldots = X_{\ell'} = 1$, where $\ell'$ is the largest even integer no more than $\ell$.  It is easy to see that
\begin{equation*}
	\label{eq:d-3}
	\textrm{supp}(\pi^{A}) = S_0^A \supset S_1^A = S_2^A \supset S_3^A = S_4^A  \supset \dotsc ,
\end{equation*}
and
\begin{equation*}
	\label{eq:d-4}
	\textrm{supp}(\pi^{B}) = S_0^B = S_1^B \supset S_2^B = S_3^B \supset S_4^B = \dotsc .
\end{equation*} 

Let $\theta = (\theta_1, \ldots, \theta_q) \in \{0, 1\}^q$ be a sequence of $q$ pull outcomes on an arm with mean $x$.  For convenience, we write 
\begin{equation}
	\label{eq:d-5}
	p(\theta\ |\ x) \triangleq \Pr_{\Theta \sim \Ber(x)^{\otimes q}}[\Theta = \theta].  
\end{equation}
We have
\begin{equation}
	\label{eq:d-6}
	p(\theta\ |\ x) = \prod_{j=1}^q x^{\theta_j} (1 - x)^{1 - \theta_j}.
\end{equation}

The following two lemmas give key properties of the sets $S_\ell^A$ and $S_\ell^B$.  Intuitively, it says that if we can only pull the arm whose mean is $x \in S_\ell^A$ (or $x \in S_\ell^B$) for a small number of times, then it is hard to differentiate its true mean $x$ from other values in $S_\ell^A$ (or $S_\ell^B$) based on the pull outcomes. 
\ifdefined\fullversion
\else
Due to the space constraints, we leave the proof of this technical lemma to the full version of this paper~\cite{KZ24}.  
\fi

\begin{lemma}
	\label{lem:key-pi}  
	For any $x\in S_\ell^A$, let $\Theta = (\Theta_1, \ldots, \Theta_q)$ be a sequence of $q \in \left[\eta^3, \frac{\eta^{2\ell-1}}{2^7}\right]$ pull outcomes on an arm with mean $x$. For any $y \in S_\ell^A\ (y \neq x)$, we have
	\begin{equation*}
		\Pr_{\Theta \sim {\Ber(x)}^{\otimes q}}\left[\frac{p(\Theta\ |\ y)}{p(\Theta\ |\ x)} < e^{-\frac{2}{\eta}} \right] \le e^{-\frac{\eta}{2^{10}}},
	\end{equation*}
	and
	\begin{equation*}
		\Pr_{ \Theta \sim {\Ber(x)}^{\otimes q}}\left[\frac{p(\Theta\ |\ y)}{p(\Theta\ |\ x)} > e^{\frac{2}{\eta}} \right] \le e^{-\frac{\eta}{2^{10}}}.
	\end{equation*}
\end{lemma}

\ifdefined\fullversion
	To prove Lemma~\ref{lem:key-pi}, we need the following lemma which concerns the differences of means in the support of $\pi$.
	\begin{lemma}
		\label{lem:mean-diff}
		For any $X = (X_1, \ldots, X_L) \in \{0,1\}^L, Y = (Y_1, \ldots, Y_L) \in \{0,1\}^L$, let $t$ be the smallest index such that $X_t \neq Y_t$. We have 
		\begin{equation*}
			\label{eq:b-1}
			\abs{\sum_{\ell=1}^L \left(\frac{X_\ell}{\eta^\ell} - \frac{Y_\ell}{\eta^\ell} \right)} = \frac{1}{\eta^t} \left(1 \pm \frac{2}{\eta} \right).
		\end{equation*}
	\end{lemma}
	
	\begin{proof}
		By the definition of $t$ and triangle inequality, we have
		\begin{eqnarray*}
			\label{eq:b-2}
			\abs{\sum_{\ell=1}^L \left(\frac{X_\ell}{\eta^\ell} - \frac{Y_\ell}{\eta^\ell} \right)} &=& \abs{\sum_{\ell=1}^L  \frac{X_\ell -Y_\ell}{\eta^\ell}} \\
			&=& \frac{1}{\eta^t} \pm \sum_{\ell=t+1}^L  \frac{X_\ell -Y_\ell}{\eta^\ell} \\
			&=& \frac{1}{\eta^t} \pm \sum_{\ell=t+1}^L  \frac{1}{\eta^\ell} =  \frac{1}{\eta^t} \left(1 \pm \frac{1/\eta}{1 - 1/\eta} \right)  \\
			&=&  \frac{1}{\eta^t} \left(1 \pm \frac{2}{\eta} \right).
		\end{eqnarray*}
	\end{proof}

	We now prove Lemma~\ref{lem:key-pi}.
	\begin{proof}
		For convenience, we introduce a random variable $Z_j \triangleq \ln\frac{p(\Theta_j | y)}{p(\Theta_j | x)}$, and try to show
		\begin{eqnarray*}
			\label{eq:e-1}
			&&\Pr_{\Theta \sim {\Ber(x)}^{\otimes q}}\left[\sum_{j=1}^q Z_j < -\frac{2}{\eta} \right] \le e^{-\frac{\eta}{2^{10}}}
			\quad \quad \text{and} \\
			&&\Pr_{\Theta \sim {\Ber(x)}^{\otimes q}}\left[\sum_{j=1}^q Z_j > \frac{2}{\eta} \right] \le e^{-\frac{\eta}{2^{10}}}.
		\end{eqnarray*} 
		
		By (\ref{eq:d-6}) we can write
		$
		Z_j = \Theta_j \ln\frac{y}{x} + (1 - \Theta_j)\ln\frac{1 - y}{1 - x}.
		$
		Given $\bE[\Theta_j] = x$, we have
		\begin{equation}
			\label{eq:e-4}
			\bE[Z_j] = x \ln\frac{y}{x} + (1 - x)\ln\frac{1 - y}{1 - x}.
		\end{equation}
		Since $\Theta_j \in \{0, 1\}$, we also have
		\begin{equation}
			\label{eq:e-5}
			\abs{Z_j} \le \max \left\{\abs{\ln \frac{y}{x}}, \abs{\ln \frac{1 - y}{1 - x}}\right\}.
		\end{equation}
		
		For any $x, y \in S_\ell^A$, by Lemma~\ref{lem:mean-diff} we have
		\begin{equation}
			\label{eq:e-6}
			\abs{x - y} \le \frac{2}{\eta^\ell} \left(1 + \frac{2}{\eta} \right).
		\end{equation}
		Using the inequality $v - v^2 \le \ln(1 + v) \le v$ for $\abs{v} \le 1/2$, Equation (\ref{eq:e-4}) can be bounded as:
		\begin{eqnarray}
			\bE[Z_j] &=& x \ln\left(1 + \frac{y - x}{x}\right) + (1 - x) \ln\left(1 + \frac{x - y}{1 - x}\right) \nonumber \\
			&\le& (y - x) + (x - y) = 0,  \label{eq:e-7}
		\end{eqnarray}
		and
		\begin{eqnarray}
			\bE[Z_j] & = & x \ln\left(1 + \frac{y - x}{x}\right) + (1 - x) \ln\left(1 + \frac{x - y}{1 - x}\right) \nonumber \\
			& \ge &  (y - x) - \frac{(y - x)^2}{x} + (x - y) - \frac{(x - y)^2}{1 - x}    \nonumber \\
			&\ge& -8 (x - y)^2  \nonumber \\
			&\ge& - \frac{2^7}{\eta^{2\ell}},  \label{eq:e-8}
		\end{eqnarray}
		where in the second inequality we have used the fact that $\frac{1}{2} \le x, y \le \frac{2}{3}$ (since $x, y \in S_\ell^A$), and in the last inequality we have used (\ref{eq:e-6}).
		
		Regarding the maximum value of $Z_j$, we have by (\ref{eq:e-6})
		\begin{equation}
			\label{eq:e-9}
			\abs{\ln \frac{y}{x}} = \abs{\ln \left({1 + \frac{y - x}{x}}\right)} \le \abs{\frac{y - x}{x}} + \abs{\frac{y - x}{x}}^2 \le \frac{2^4}{\eta^\ell},
		\end{equation}
		and
		\begin{equation}
			\label{eq:e-10}
			\abs{\ln \frac{1 - y}{1 - x}} = \abs{\ln \left({1 + \frac{x - y}{1 - x}}\right)} \le \abs{\frac{y - x}{1 - x}} + \abs{\frac{y - x}{1 - x}}^2 \le \frac{2^4}{\eta^\ell}.
		\end{equation}
		Combining (\ref{eq:e-5}), (\ref{eq:e-9}), (\ref{eq:e-10}), we have 
		\begin{equation}
			\label{eq:e-12}
			\abs{Z_j} \le \frac{2^4}{\eta^\ell}.
		\end{equation}
		
		By (\ref{eq:e-12}), (\ref{eq:e-7}), and (\ref{eq:e-8}), using Chernoff-Hoeffding inequality (Lemma~\ref{lem:chernoff}) and setting $t = 1/\eta$ and $M = 2^4/\eta^\ell$,  we have
		\begin{equation*}
			\Pr_{\Theta \sim {\Ber(x)}^{\otimes q}}\left[\sum_{j = 1}^{q} Z_j \ge \frac{2}{\eta} \right]  \le e^{-\frac{2q \eta^{-2}}{M^2}} = e^{-\frac{q \eta^{-2}}{2^{7} \eta^{-2\ell}}} \le e^{-\frac{\eta}{2^{10}}},
		\end{equation*}
		and
		\begin{eqnarray*}
			\Pr_{\Theta \sim {\Ber(x)}^{\otimes q}}\left[\sum_{j = 1}^{q} Z_j \le - \frac{2}{\eta} \right] &=&	\Pr_{\Theta \sim {\Ber(x)}^{\otimes q}}\left[\sum_{j = 1}^{q} Z_j \le - q \cdot \frac{2^7}{\eta^{2\ell}} - \frac{1}{\eta} \right] \\
			&\le& e^{-\frac{q \eta^{-2}}{2 M^2}} \le e^{-\frac{\eta}{2^{10}}},
		\end{eqnarray*}
		where we have used the fact that $q \in \left[\eta^3, \frac{\eta^{2\ell-1}}{2^7}\right]$.
	\end{proof}
\else
\fi

The following lemma is symmetric to Lemma~\ref{lem:key-pi}, and can be proved using a similar line of arguments.

\begin{lemmaprime}{lem:key-pi}\label{lem:key-pi-2}
	For any $x \in S_\ell^B$, let $\Theta = (\Theta_1, \ldots, \Theta_q)$ be a sequence of $q \in [\eta^3, \frac{\eta^{2\ell -1}}{2^7}]$ pull outcomes on an arm with mean $x$. For any $y \in S_\ell^B\ (y \neq x)$, we have
	\begin{equation}
		\Pr_{\Theta \sim \Ber(x)^{\otimes q}} \left[\frac{p(\Theta \mid y)}{p(\Theta \mid x)} < e^{-\frac{2}{\eta}}\right] \le e^{-\frac{\eta}{2^{10}}},
	\end{equation}
	and 
	\begin{equation}
		\Pr_{\Theta \sim \Ber(x)^{\otimes q}}\left[\frac{p(\Theta \mid y)}{p(\Theta \mid x)} > e^{\frac{2}{\eta}}\right] \le e^{-\frac{\eta}{2^{10}}}.
	\end{equation}	
\end{lemmaprime}

\vspace{2mm}
\noindent{\bf Instance Complexity under Distribution $\pi^{\otimes n}$.\ \ }
We now try to bound the instance complexity of an input sampled from distribution $\pi^{\otimes n}$.
Let $\mu_* = \frac{1}{2} + \sum_{\ell=1}^L \frac{1}{\eta^\ell}$. The following event stands for the case when there is only one best arm with mean $\mu_*$.
\begin{equation}
	\E_0: \ \exists \text{\ unique\ } i^* \in [n] \ \ \text{s.t.} \ \mu_{i^*} = \mu_*.
\end{equation}

The following lemma shows that $\E_0$ holds with at least a constant probability.

\begin{lemma}
	\label{lem:E}
	$\Pr_{(\mu_1, \ldots, \mu_n) \sim \pi^{\otimes n}}[\E_0] \ge 1/e$.
\end{lemma}

\begin{proof}
	We have
	\begin{eqnarray*}
		\Pr_{(\mu_1, \ldots, \mu_n) \sim \pi^{\otimes n}}[\E_0] &=& \sum_{i = 1}^{n} \left( \Pr[\mu_i = \mu_*] \prod_{j \in [n], j \neq i}\Pr[\mu_j \neq \mu_*] \right) \\
		&=& n \cdot  \frac{1}{\eta^{2L}} \cdot \left(1 - \frac{1}{\eta^{2L}}\right)^{n - 1} \\
		&=& \left(1 - \frac{1}{n}\right)^{n - 1} \ge \frac{1}{e}\ .
	\end{eqnarray*}
\end{proof}

We now try to upper bound the instance complexity of inputs sampled from distribution $\pi^{\otimes n}$, conditioned on event $\E_0$.

\begin{lemma}
	\label{lem:H}
	$\bE_{(\mu_1, \ldots, \mu_n) \sim \pi^{\otimes n}}[H\ |\ \E_0] \le \eta^{2+2L} L$.
\end{lemma}

\begin{proof}
	Conditioned on $\E_0$, let $i^*$ be the unique best arm with mean $\mu_*$.  We can write 
	\begin{eqnarray}
		&& \bE_{(\mu_1, \ldots, \mu_n) \sim \pi^{\otimes n}}[H \mid \E_0] \nonumber \\
		&=& \bE_{(\mu_1, \ldots, \mu_n) \sim \pi^{\otimes n}}\left[\sum_{i \in [n], i \neq i^*} (\mu_* - \mu_i)^{-2} \left| \max_{i \neq i^*}\{\mu_i\} < \mu_*\right.\right] \nonumber \\
		&=& (n - 1) \bE_{\mu \sim \pi} \left[(\mu_* - \mu)^{-2} \mid \mu < \mu_*\right].  \label{eq:c-1}
	\end{eqnarray}
	
	To upper bound (\ref{eq:c-1}), we partition the values in $\mathrm{supp}(\pi)$ into $L$ disjoint sets.  For each $\ell \in [L]$, we define
	\begin{equation}
		\label{eq:c-2}
		P_\ell \triangleq \left\{\left.\frac{1}{2} + \sum_{k = 1}^{\ell-1} \frac{1}{\eta^{k}} + \frac{0}{\eta^\ell} + \sum_{k=\ell+1}^{L} \frac{X_k}{\eta^k}\ \right|\ (X_\ell, \ldots, X_L) \in \{0,1\}^{L-\ell+1} \right\}.		
	\end{equation}
	Clearly, we have $\bigcup_{\ell=1}^{L} P_\ell = \mathrm{supp}(\pi) \setminus \{\mu_*\}$, and for any $\ell \in [L]$ and any $\mu \in P_\ell$,
	\begin{equation}
		\label{eq:c-3}
		\mu_* - \mu \ge \frac{1}{\eta^{\ell}}.
	\end{equation}
	Plugging (\ref{eq:c-3}) to (\ref{eq:c-1}), we have
	\begin{eqnarray*}
		(\ref{eq:c-1}) &\le& (n-1) \cdot \sum_{\ell=1}^L \left(\Pr_{\mu \sim \pi} \left[\mu \in P_\ell\ |\ \mu < \mu_*\right] \cdot \eta^{2\ell} \right) \\
		&=& (n-1) \cdot \sum_{\ell=1}^L \left(\frac{\Pr_{\mu \sim \pi}\left[\mu \in P_\ell, \mu < \mu_*\right]}{\Pr_{\mu \sim \pi}[\mu < \mu_*]} \cdot \eta^{2\ell} \right) \\
		&=& (n-1) \cdot \sum_{\ell=1}^L  \left(\frac{\left(\frac{1}{\eta^2}\right)^{\ell-1} \left(1 - \frac{1}{\eta^2}\right)}{1 - \frac{1}{n}} \cdot \eta^{2\ell} \right) \\
		&=& n \cdot \left(1 - \frac{1}{\eta^2}\right) \cdot \sum_{\ell=1}^L \eta^2 \\
		&\le& \eta^2 L n = \eta^{2+2L} L.
	\end{eqnarray*}
\end{proof}

Define event 
\begin{equation}
	\label{eq:event-E-1}
	\E_1: \E_0 \ \text{holds} \  \wedge (H < 2\eta^{2+2L} L). 
\end{equation}

By Markov's inequality and Lemma~\ref{lem:H}, we have $$\Pr[H \ge 2 \eta^{2+2L}L\ |\ \E_0] \le 1/2,$$ which, combined with Lemma~\ref{lem:E}, gives the following lemma.
\begin{lemma}
	\label{lem:E-1}
	$\Pr[\E_1] \ge 1/(2e)$.
\end{lemma}

\vspace{2mm}
\noindent{\bf Hard Input Distribution.\ \ }  The hard input distribution we use for proving the lower bound is $(\pi^{\otimes n}\ |\ \E_1)$.  That is, the probability mass is uniformly distributed among the support of $\pi^{\otimes n}$ {\em except} those instances in which there is $0$ or multiple arms with means $\mu_*$ (i.e., when $\E_0$ does {\em not} hold) and those of which the instance complexity is more than $2\eta^{2+2L} L$.  

In our lower bound proof in Section~\ref{sec:lb-two}, we will spend most of our time working on the input distribution $\pi^{\otimes n}$.  We will switch to $(\pi^{\otimes n}\ |\ \E_1)$ at the end of the proof.

\subsection{Classes of Hard Distributions}
\label{sec:class-distribution}

In this section, we define the classes of hard distributions that we use for the generalized round elimination.
We start by introducing a concept called {\em good pull outcome sequences}.

\vspace{2mm}
\noindent{\bf Good Pull Outcome Sequences.\ \ }
We say a pull outcome sequence $\theta = (\theta_1, \ldots, \theta_q)$ {\em good} w.r.t.\ $S_\ell^A$ if for any $x, y \in S_\ell^A$, $p(\theta| x)$ and $p(\theta| y)$ are close.  More precisely, we define the set of good pull outcome sequences w.r.t.\ $S_\ell^A$ as
\begin{equation}
	\label{eq:f-1}
	G_\ell^A \triangleq \left\{\theta\ \left|\ \forall{x, y \in S_\ell^A} :  \frac{p(\theta \mid x)}{p(\theta \mid y)} \in \left[e^{-\frac{4}{\eta}}, e^{\frac{4}{\eta}}\right] \right.\right\}.
\end{equation}
Similarly, we define the set of good pull outcome sequences w.r.t.\ $S_\ell^B$ as
\begin{equation}
	\label{eq:f-2}
	G_\ell^B \triangleq \left\{\theta\ \left|\ \forall{x, y \in S_\ell^B} :  \frac{p(\theta \mid x)}{p(\theta \mid y)} \in \left[e^{-\frac{4}{\eta}}, e^{\frac{4}{\eta}}\right] \right.\right\}.
\end{equation}

The following lemma says that when the length of sequence $q$ is not large, then with high probability, the pull outcome sequence is good. 

\begin{lemma}
	\label{lem:good-pull-outcome}
	Let $\Theta = (\Theta_1, \ldots, \Theta_q)$ be a sequence of $q \in [\eta^3, \frac{\eta^{2\ell-1}}{2^7}]$ pull outcomes on an arm with mean $\mu^A$.  For any distribution $\sigma^A$ with support $S_\ell^A$, 
	we have
		$	\Pr_{\mu^A \sim \sigma^A, \Theta \sim {\Ber(\mu^A)}^{\otimes q}}\left[\Theta \not \in G_\ell^A \ \right] \le n^{-10}.$
\end{lemma}

\begin{proof}
	By the law of total probability, we write
	\begin{eqnarray}
		&&\Pr_{\mu^A \sim \sigma^A, \Theta \sim {\Ber(\mu^A)}^{\otimes q}}\left[\Theta \not \in G_\ell^A \right] \nonumber \\
		&=& \sum_{z \in S_\ell^A} \left( \Pr_{\Theta \sim {\Ber(z)}^{\otimes q}}\left[\Theta \not\in G_\ell^A\right] \Pr_{\mu^A \sim \sigma^A}[\mu^A = z] \right).
		\label{eq:g-3}
	\end{eqnarray}
	By definition, $\theta \not\in G_\ell^A$ if and only if there exists a pair of $x, y \in S_\ell^A$ such that
	\begin{equation}
		\label{eq:g-4}
		\frac{p(\theta\ |\ x)}{p(\theta\ |\ y)} > e^{\frac{4}{\eta}} \quad \text{or} \quad \frac{p(\theta\ |\ x)}{p(\theta\ |\ y)} < e^{-\frac{4}{\eta}}.
	\end{equation}
	
	We first consider the case $\frac{p(\theta|x)}{p(\theta|y)} > e^{\frac{4}{\eta}}$.  In this case, for any $z \in S_\ell^A$, we have 
	\begin{equation*}
		\label{eq:g-5}
		\frac{p(\theta\ |\ x)}{p(\theta\ |\ z)} > e^{\frac{2}{\eta}} \quad \text{or} \quad  \frac{p(\theta\ |\ y)}{p(\theta\ |\ z)} < e^{-\frac{2}{\eta}} .
	\end{equation*}
	Consequently,
	\begin{eqnarray}
		&&\Pr_{\Theta \sim {\Ber(z)}^{\otimes q}}\left[\frac{p(\Theta\ |\ x)}{p(\Theta\ |\ y)} > e^{\frac{4}{\eta}}\right] \nonumber \\
		&\le& \Pr_{\Theta \sim {\Ber(z)}^{\otimes q}}\left[\frac{p(\Theta\ |\ x)}{p(\Theta\ |\ z)} > e^{\frac{2}{\eta}}\right] + \Pr_{\Theta \sim {\Ber(z)}^{\otimes q}}\left[\frac{p(\Theta\ |\ y)}{p(\Theta\ |\ z)} < e^{-\frac{2}{\eta}}\right] \nonumber \\
		&\le& 2 e^{-\frac{\eta}{2^{10}}},  \label{eq:g-6}
	\end{eqnarray}
	where in the last inequality we have used Lemma~\ref{lem:key-pi}.
	
	By a similar argument, we can show
	\begin{equation}
		\label{eq:g-61}
		\Pr_{\Theta \sim {\Ber(z)}^{\otimes q}}\left[\frac{p(\Theta\ |\ x)}{p(\Theta\ |\ y)} < e^{-\frac{4}{\eta}}\right]
		\le 2 e^{-\frac{\eta}{2^{10}}}.
	\end{equation}
	
	By (\ref{eq:g-6}), (\ref{eq:g-61}), and the definition of $G_\ell^A$ in (\ref{eq:f-1}), we have
	\begin{eqnarray}
		&&\Pr_{\Theta \sim {\Ber(z)}^{\otimes q}}\left[\Theta \not\in G_\ell^A\right] \nonumber \\
		&\le& \sum_{x, y \in S_\ell^A} \left(\Pr_{\Theta \sim {\Ber(z)}^{\otimes q}} \left[\frac{p(\Theta\ |\ x)}{p(\Theta\ |\ y)} > e^{\frac{4}{\eta}}\right]\right. + \nonumber \\
		&& \left.\Pr_{\Theta \sim {\Ber(z)}^{\otimes q}} \left[\frac{p(\Theta\ |\ x)}{p(\Theta\ |\ y)} < e^{-\frac{4}{\eta}} \right] \right) \nonumber \\
		&\le& 4 \abs{S_\ell^A}^2 e^{-\frac{\eta}{2^{10}}}, \label{eq:g-7}
	\end{eqnarray}
	where in the last inequality we have taken a union bound on all pairs $(x, y) \in S_\ell^A \times S_\ell^A$.
	
	Plugging (\ref{eq:g-7}) to (\ref{eq:g-3}), we have
	\begin{eqnarray*}
		&&\Pr_{\mu^A \sim \sigma^A, \Theta \sim {\Ber(\mu^A)}^{\otimes q}}\left[\Theta \not\in G_\ell^A\right] \nonumber \\
		&\le& \sum_{z \in S_\ell^A} \left(4 \abs{S_\ell^A}^2 e^{-\frac{\eta}{2^{10}}} \Pr_{\mu^A \sim \sigma^A}[\mu^A = z] \right) \nonumber \\
		&=& 4 \abs{S_\ell^A}^2 e^{-\frac{\eta}{2^{10}}}  \le n^{-10},
		\label{eq:g-8}
	\end{eqnarray*}
	where the last inequality is due to $\eta \ge \log^2 n$.
\end{proof}

The following lemma is symmetric to Lemma~\ref{lem:good-pull-outcome}, and can be proved using a similar line of arguments.

\begin{lemmaprime}{lem:good-pull-outcome}
	\label{lem:good-pull-outcome-2}
	Let $\Theta = (\Theta_1, \ldots, \Theta_q)$ be a sequence of $q \in [\eta^3, \frac{\eta^{2\ell-1}}{2^7}]$ pull outcomes on an arm with mean $\mu^B$.  For any distribution $\sigma^B$ with support $S_\ell^B$, 
	we have
		$	\Pr_{\mu^B \sim \sigma^B, \Theta \sim {\Ber(\mu^B)}^{\otimes q}}\left[\Theta \not \in G_\ell^B \ \right] \le n^{-10}.$
\end{lemmaprime}

\vspace{2mm}
\noindent{\bf Classes of Distributions $\D_\ell^A$, $\D_\ell^B$, and $\D_\ell\ (\ell = 0, 1, \ldots, L)$.\ \ }
We are now ready to define classes of input distributions on which we will perform the induction.   

For $\ell \in \{0, 1, \ldots, L\}$, we define $\D_\ell^A$ to be the class of distributions $\sigma^A$ with support $S_\ell^A$ such that
\begin{equation}
	\label{eq:h-1}
	\forall{x, y\in S_\ell^A} : \frac{\Pr_{\mu^A \sim \sigma^A}[\mu^A = x]}{\Pr_{\mu^A \sim \sigma^A}[\mu^A = y]} = \frac{\Pr_{\mu^A \sim \pi^A}\left[ \mu^A = x \right]}{\Pr_{\mu^A \sim \pi^A}\left[ \mu^A = y \right]} \cdot e^{\pm \frac{4\ell}{\eta}} .
\end{equation}
Similarly, we define $\D_\ell^B$ to be the class of distributions $\sigma^B$ with support $S_\ell^B$ such that
\begin{equation}
	\label{eq:h-2}
	\forall{x, y\in S_\ell^B} : \frac{\Pr_{\mu^B \sim \sigma^B}[\mu^B = x]}{\Pr_{\mu^B \sim \sigma^B}[\mu^B = y]} = \frac{\Pr_{\mu^B \sim \pi^B}\left[ \mu^B = x \right]}{\Pr_{\mu^B \sim \pi^B}\left[ \mu^B = y \right]} \cdot e^{\pm \frac{4\ell}{\eta}} .
\end{equation}
Let $\D_\ell = (\D_\ell^A, \D_\ell^B)$.  We say a distribution $\sigma = (\sigma^A, \sigma^B) \in \D_\ell$ iff $\sigma^A \in \D_\ell^A$ and $\sigma^B \in \D_\ell^B$. 

We have the following simple fact.
\begin{fact}
	\label{fact:base-distribution}
	$\D_0^A = \{\pi^A\}$, $\D_0^B = \{\pi^B\}$, and $\D_0 = \{\pi\}$.
\end{fact}

The following lemma shows a key property of distribution classes $\D_\ell^A$.  Intuitively, if the mean of an arm follows a distribution $\sigma^A \in \D_\ell^A$, then after observing a good sequence of pulls that belongs to $G_k^A$ for a $k \ge \ell + 1$,  the posterior distribution of the arm belongs to distribution class $\D_k^A$.   

\begin{lemma}
	\label{lem:distribution-class}
	For any $\ell \in \{0, 1, \ldots, L-1\}$, any $k \in \{\ell+1, \ldots, L\}$, any distribution $\sigma^A \in \D_\ell^A$, and any good sequence of pull outcomes $\theta = (\theta_1, \ldots, \theta_q) \in G_k^A$, the posterior distribution of $\sigma^A$ after observing a sequence of pull outcomes being $\theta$ and conditioning on the mean of the arm $\mu^A \in S_k^A$, denoted by $(\sigma^A\ |\ \theta, \mu^A \in S_k^A)$, belongs to the distribution class $\D_k^A$.
\end{lemma}

\begin{proof}
	Fix two arbitrary fixed values $x, y \in S_k^A$. By Bayes' theorem, we have
	\begin{eqnarray}
		&&\Pr_{\mu^A \sim \sigma^A, \Theta \sim \Ber(\mu^A)^{\otimes q}}[\mu^A = x \mid \Theta = \theta, \mu^A \in S_k^A] \nonumber \\
		&=&  \frac{\displaystyle\Pr_{\mu^A \sim \sigma^A, \Theta \sim \Ber(\mu^A)^{\otimes q}}[\Theta = \theta, \mu^A \in S_k^A \mid \mu^A = x] \Pr_{\mu^A \sim \sigma^A}[\mu^A = x]}{\displaystyle\Pr_{\mu^A \sim \sigma^A, \Theta \sim \Ber(\mu^A)^{\otimes q}}[\Theta = \theta, \mu^A \in S_k^A]} \nonumber \\ 
		&=&  \frac{\Pr_{\Theta \sim \Ber(x)^{\otimes q}}[\Theta = \theta] \Pr_{\mu^A \sim \sigma^A}[\mu^A = x]}{\Pr_{\mu^A \sim \sigma^A, \Theta \sim \Ber(\mu^A)^{\otimes q}}[\Theta = \theta, \mu^A \in S_k^A]} \nonumber \\
		& =& \frac{p(\theta\ |\ x) \cdot \Pr_{\mu^A \sim \sigma^A}[\mu^A = x]}{\Pr_{\mu^A \sim \sigma^A, \Theta \sim \Ber(\mu^A)^{\otimes q}}[\Theta = \theta, \mu^A \in S_k^A]}, \label{eq:i-1}
	\end{eqnarray}
	where in the second equality we have used  the fact $\mu^A = x \in S_k^A$, and in the third equality we have used the definition of $p(\theta|x)$ in (\ref{eq:d-5}).
	
	Similarly, we have
	\begin{eqnarray*}
		&&\Pr_{\mu^A \sim \sigma^A, \Theta \sim \Ber(\mu^A)^{\otimes q}}[\mu^A = y \mid \Theta = \theta, \mu^A \in S_k^A] \\
		&=&  \frac{p(\theta\ |\ y) \cdot \Pr_{\mu^A \sim \sigma^A}[\mu^A = y]}{\Pr_{\mu^A \sim \sigma^A, \Theta \sim \Ber(\mu^A)^{\otimes q}}[\Theta = \theta, \mu^A \in S_k^A]}.
		\label{eq:i-2}
	\end{eqnarray*}
	We next have
	\begin{eqnarray}
		&&\frac{\Pr_{\mu^A \sim \sigma^A, \Theta \sim \Ber(\mu^A)^{\otimes q}}[\mu^A = x \mid \Theta = \theta, \mu^A \in S_k^A] }{\Pr_{\mu^A \sim \sigma^A, \Theta \sim \Ber(\mu^A)^{\otimes q}}[\mu^A = y \mid \Theta = \theta, \mu^A \in S_k^A]} \nonumber \\
		&\stackrel{(\ref{eq:i-1}), (\ref{eq:i-2})}{=}& \frac{\Pr_{\mu^A \sim \sigma^A}[\mu^A = x]}{\Pr_{\mu^A \sim \sigma^A}[\mu^A = y]} \cdot \frac{p(\theta\ |\ x)}{p(\theta\ |\ y)} \label{eq:i-3} \\
		&=& \frac{\Pr_{\mu^A \sim \pi^A}[\mu^A = x]}{\Pr_{\mu^A \sim \pi^A}[\mu^A = y]} \cdot e^{\pm \frac{4\ell}{\eta}} \cdot e^{\pm \frac{4}{\eta}}  \label{eq:i-4} \\
		&=& \frac{\Pr_{\mu^A \sim \pi^A}[\mu^A = x]}{\Pr_{\mu^A \sim \pi^A}[\mu^A = y]} \cdot e^{\pm \frac{4k}{\eta}},  \label{eq:i-6}
	\end{eqnarray}
	where from (\ref{eq:i-3}) to (\ref{eq:i-4}) we have used the definition of distribution class $\D_\ell^A$ in (\ref{eq:h-1}) and the fact $S_k^A \subseteq S_\ell^A$, as well as the definition of $G_{\ell}^A$ and the fact $\theta \in G_k^A$.  From (\ref{eq:i-4}) to (\ref{eq:i-6}) we have used the fact $k \ge \ell+1$.
	
	By (\ref{eq:i-6}) and the definition of $\D_k^A$ in (\ref{eq:h-1}), we know that $(\sigma^A\ |\ \theta, \mu^A \in S_k^A) \in \D_k^A$.
\end{proof}

The following lemma is symmetric to Lemma~\ref{lem:distribution-class}, and can be proved using a similar line of arguments.

\begin{lemmaprime}{lem:distribution-class}
	\label{lem:distribution-class-2}
	For any $\ell \in \{0, 1, \ldots, L-1\}$, any $k \in \{\ell+1, \ldots, L\}$, any distribution $\sigma^B \in \D_\ell^B$, and any good sequence of pull outcomes $\theta = (\theta_1, \ldots, \theta_q) \in G_k^B$, the posterior distribution of $\sigma^B$ after observing a sequence of pull outcomes being $\theta$ and conditioning on the mean of the arm $\mu^B \in S_k^B$, denoted by $(\sigma^B\ |\ \theta, \mu^B \in S_k^B)$, belongs to the distribution class $\D_k^B$.
\end{lemmaprime}

\medskip
Let 
$$\mu_*^A = \frac{1}{2} + 2 \sum_{\ell : 1 \le 2\ell+1 \le L} \frac{1}{\eta^{2\ell+1}}$$ be the mean of local best arm at Alice's side, and let $$\mu_*^B = \frac{1}{2} + 2 \sum_{\ell : 1 \le 2\ell \le L} \frac{1}{\eta^{2\ell}}$$ be the mean of local best arm at Bob's side. 
The following lemma shows that an arm whose mean is distributed according to $\sigma \in \D_\ell^A$ has a small probability being a local best arm.  

\begin{lemma}
	\label{lem:best-local-arm}
	For any $\ell \in \{0, 1, \ldots, L\}$, and any $\sigma^A \in \D_\ell^A$, we have 
	$\Pr_{\mu^A \sim \sigma^A}[\mu^A = \mu_*^A] \le e^{\frac{4\ell}{\eta}} \eta^{-2d_1},$
	where $d_1 = \abs{\{k \in \mathbb{Z} \ |\ \ell < 2k+1 \le L\}}$ is the number of odd integers in the set $\{\ell+1, \ldots, L\}$.
\end{lemma}

\begin{proof}
	We first define a few quantities. Let
	\begin{equation}
		\label{eq:j-3}
		\rho_{\max} = \max_{x \in S_\ell^A} \frac{\Pr_{\mu^A \sim \sigma^A}[\mu^A = x]}{\Pr_{\mu^A \sim \pi^A}[\mu^A = x \ |\ \mu^A \in S_\ell^A]},
	\end{equation}
	and slightly abusing the notation, let $x \in S_\ell^A$ be the value that achieves $\rho_{\max}$.  Let
	\begin{equation}
		\label{eq:j-4}
		\rho_{\min} = \min_{y \in S_\ell^A} \frac{\Pr_{\mu^A \sim \sigma^A}[\mu^A = y]}{\Pr_{\mu^A \sim \pi^A}[\mu^A = y \ |\ \mu^A \in S_\ell^A]},
	\end{equation}
	and let $y \in S_\ell^A$ be the value that achieves $\rho_{\min}$.
	It is clear that $\rho_{\min} \le 1 \le \rho_{\max}$. We also have
	\begin{eqnarray}
		\frac{\rho_{\max}}{\rho_{\min}} &=& \frac{\Pr_{\mu^A \sim \sigma^A}[\mu^A = x]}{\Pr_{\mu^A \sim \sigma^A}[\mu^A = y]} \cdot \frac{\Pr_{\mu^A \sim \pi^A}[\mu^A = y \ |\ \mu^A \in S_\ell^A]}{\Pr_{\mu^A \sim \pi^A}[\mu^A = x \ |\ \mu^A \in S_\ell^A]} \nonumber \\
		&=& \frac{\Pr_{\mu^A \sim \sigma^A}[\mu^A = x]}{\Pr_{\mu^A \sim \sigma^A}[\mu^A = y]} \cdot \frac{\Pr_{\mu^A \sim \pi^A}[\mu^A = y]}{\Pr_{\mu^A \sim \pi^A}[\mu^A = x]} \quad (\text{since}\ x, y \in S_\ell^A) \nonumber \\
		&=& e^{\pm \frac{4\ell}{\eta}}.  \quad  (\text{since}\ \sigma^A \in \D_\ell^A)  
	\end{eqnarray}
	We thus have  $e^{- \frac{4\ell}{\eta}} \le \rho_{\min} \le 1 \le \rho_{\max} \le e^{ \frac{4\ell}{\eta}}$. The last inequality $\rho_{\max} \le e^{ \frac{4\ell}{\eta}}$ implies
	\begin{eqnarray}
		\Pr_{\mu^A \in \sigma^A}[\mu^A = \mu_*^A] &\le& \Pr_{\mu^A \sim \pi^A}\left[\left.\mu^A = \mu_*^A\ \right|\ \mu^A \in S_\ell^A\right] \cdot e^{\frac{4\ell}{\eta}} \nonumber \\
		&=& \left(\frac{1}{\eta^2}\right)^{d_1} \cdot e^{\frac{4\ell}{\eta}},
	\end{eqnarray}
	where $d_1 = \{k \in \mathbb{Z} \ |\ \ell < 2k+1 \le L\}$.
\end{proof}

The following lemma is similar to Lemma~\ref{lem:best-local-arm}, and can be proved using a similar line of arguments.

\begin{lemmaprime}{lem:best-local-arm}
	\label{lem:best-local-arm-2}
	For any $\ell \in \{0, 1, \ldots, L\}$, and any $\sigma^B \in \D_\ell^B$, we have 
	$\Pr_{\mu^B \sim \sigma^B}[\mu^B = \mu_*^B] \le e^{\frac{4\ell}{\eta}} \eta^{-2d_0},$
	where $d_0 = \abs{\{k \in \mathbb{Z} \ |\ \ell < 2k \le L\}}$ is the number of even integers in the set $\{\ell+1, \ldots, L\}$.
\end{lemmaprime}

\subsection{The Lower Bound for $K = 2$}
\label{sec:lb-two}

In this section, we show the following lower bound result for the case of two agents.
\begin{theorem}
	\label{thm:lb-two}
	For any $1 \le R \le \frac{\log n}{24\log\log n}$, any $R$-round $2$-agent algorithm that solves $n$-arm BAI in the heterogeneous CL model with probability $0.99$ needs to use at least $H n^{\frac{1}{25R}}$ time.
\end{theorem}

By Yao's Minimax Lemma, we can just prove for any deterministic algorithm over the hard input distribution $(\pi^{\otimes n}\ |\ \E_1)$.  

We will first analyze the success probability of any deterministic algorithm $\A$ on input distribution $\pi^{\otimes n}$.
We say $\A$ succeeds on an input instance $I$ if $\A$ outputs an index $i$ such that $\mu_i = \mu_*$. Note that there could be multiple $i \in [n]$ such that $\mu_i = \mu_*$ and $\A$ can output any index in this set.

\vspace{2mm}
\noindent{\bf The Induction Step.\ \ }
\label{sec:induction}
Let the quantity $\lambda_r$ be the largest success probability of a $(R - r)$-round $2\zeta \eta^{2+2L} L$-time algorithm on some input distribution in $\D_{6r}^{\otimes \kappa}$ for some $\kappa \in [n]$.  That is, 
\begin{equation}
	\label{eq:j-6}
	\lambda_r \triangleq \max_{\kappa \in [n]} \max_{\nu \in \D_{6r}^{\otimes \kappa}} 
	\max_{\A} 
	\Pr_{I \sim \nu}[\A \ \text{succeeds on}\ I],
\end{equation}
where $\max_{\A}$ runs over all algorithms $\A$ that use $(R - r)$ rounds and $2\zeta \eta^{2+2L} L$ time.

The following lemma connects the error probabilities $\lambda_r$ and $\lambda_{r+1}$, and is the key for the induction.

\begin{lemma}
	\label{lem:induction}
	For any $r = 1, \ldots, R-1$, it holds that $$\lambda_r \le  \lambda_{r+1} + 4e^{\frac{10L}{\eta}} L^2 \eta^{-\frac{5}{2}} + n^{-5}.$$
\end{lemma}

The rest of Section~\ref{sec:induction} devotes to the proof of Lemma~\ref{lem:induction}. 
Slightly abusing the notation, let $\kappa$ be the value that maximizes the error in the definition of $\lambda_r$ (the first $\max$ in (\ref{eq:j-6})).
We write $\nu = (\sigma_1^A, \ldots, \sigma_\kappa^A, \sigma_1^B, \ldots, \sigma_\kappa^B)$. 
Since 
$\nu \sim \D_{6r}^{\otimes \kappa} = \left((\D_{6r}^A)^{\otimes \kappa}, (\D_{6r}^B)^{\otimes \kappa}\right)$, we have for any $i \in [\kappa]$, $\sigma_i^A \in \D_{6r}^A$ and $\sigma_i^B \in \D_{6r}^B$.

Consider the {\em first round} of the collaborative learning process.
Let random variables $\W^A$ and $\W^B$ be the pull history (i.e., the sequence of $\langle$arm index, pull outcome$\rangle$ pairs) of Alice and Bob, respectively.  Let random variables $\Theta_i^A$ and $\Theta_i^B$  be the sequence of pull outcomes in the pull history $\W^A$ and $\W^B$ projecting on arm $i$, respectively.  Let $t_i^A$ be the number of pulls Alice makes on arm $i$, and let $t_i^B$ be the number of pulls Bob makes on arm $i$.  

For $\ell = \{0, 1, \ldots, L\}$, we introduce the following sets of arms. 
\begin{eqnarray}
	E_\ell^A & = & \{i\ |\ \gamma \eta^{2(\ell-1)} < t_i^A \le \gamma \eta^{2\ell}\}\ , \label{eq:j-7} \\
	E_\ell^B & = & \{i\ |\ \gamma \eta^{2(\ell-1)} < t_i^B \le \gamma \eta^{2\ell}\}\ . \label{eq:j-8}
\end{eqnarray}

To facilitate the analysis, we augment the algorithm after the first round of pulls by publishing a set of arms, as well as making some additional pulls on the remaining arms so as to massage the posterior mean distribution.  By publishing arm $i$ we mean revealing its local means $\mu_i^A$ and $\mu_i^B$ (and thus also its global mean $\mu_i = (\mu_i^A + \mu_i^B)/2$) to both Alice and Bob. We remove arm $i$ from the set of arms if $\mu_i \neq \mu_*$, otherwise we just output arm $i$ and be done. Note that such an augmentation only leads to a stronger lower bound, since the success probability of the augmented algorithm can only increase compared with the algorithm before the augmentation.  We also include all additional pulls to $\W^A$ and $\W^B$.

\vspace{2mm}
\noindent{\bf Arm Publishing and Additional Pulls}

\noindent\rule{8.5cm}{0.4pt}

\begin{enumerate}
	\item Publish all arms in the following set:
	\begin{eqnarray*}
		&&E = E^A \cup E^B \\
		&\text{where}& \quad  E^A = \bigcup_{\ell=6(r+1)}^\infty E_\ell^A \quad \text{and} \quad E^B = \bigcup_{\ell=6(r+1)}^\infty E_\ell^B. \label{eq:set-E}
	\end{eqnarray*}
	
	\item For each arm $i \in [\kappa] \backslash E$, Alice makes additional pulls on it until her number of pulls on arm $i$ reaches $\gamma \eta^{2(6(r+1)-1)}$, and Bob makes additional pulls on it until his number of pulls on arm $i$ reaches $\gamma \eta^{2(6(r+1)-1)}$.
	
	\item Let $P^A = \left\{i\ \left|\ \mu_i^A \in S_{6(r+1)}^A\right.\right\}$, and $P^B = \left\{i\ \left|\ \mu_i^B \in S_{6(r+1)}^B\right.\right\}$.  Publish all arms in $[\kappa] \backslash (P^A \cap P^B)$.
\end{enumerate} 
\noindent\rule{8.5cm}{0.4pt}
\vspace{1mm}

Let $T = \{i \in [\kappa] \ |\ \mu_i = \mu_*\}$ be the set of best arms. 
We try to analyze the probability that the augmented algorithm correctly outputs an arm in $T$, which is upper bounded by the sum of the probabilities of the following three events: 
\begin{enumerate}
	\item $T \cap E^A \neq \emptyset$.
	
	\item  $T \cap E^B \neq \emptyset$.
	
	\item $\tilde{\A}$ succeeds on $\left(P^A \cap P^B\right) \backslash E$, where $\tilde{\A}$ is the $(R-(r+1))$-round algorithm obtained from $\A$ conditioned on the pull history of the first round being $\W^A$ and $\W^B$.  
\end{enumerate}

The following lemma upper bounds the first probability.  
\ifdefined\fullversion
\else
Its proof is quite technical and lengthy; due to space constraints, we leave it to the full version of this paper~\cite{KZ24}.
\fi

\begin{lemma}
	\label{lem:publish-alice}
		$	\Pr_{I \sim \nu, \W^A, \W^B}\left[T \cap E^A \neq \emptyset\right] \le  2e^{\frac{10L}{\eta}} L^2 \eta^{-\frac{5}{2}} + n^{-6}.$
\end{lemma}

\ifdefined\fullversion
	\begin{proof}
		Since $T$ is fully determined by the input instance $I$, and $E^A$ is fully determined by $\W^A$, we can omit $\W^B$ when measuring the probability.  By a union bound we have
		\begin{equation}
			\label{eq:l-1}
			\Pr_{I \sim \nu, \W^A}\left[T \cap E^A \neq \emptyset\right]  \le \sum_{\ell =6 (r + 1)}^{\infty} \Pr_{I \sim \nu, \W^A}\left[T \cap E^A_\ell  \neq \emptyset \right].
		\end{equation}
		By a counting argument,
		\begin{equation}
			\label{eq:l-2}
			\abs{E^{A}_\ell} \le \frac{2 \zeta \eta^{2 + 2L} L}{\gamma \eta^{2(\ell - 1)}} \le 2 \eta^{-\frac{1}{2}}\eta^{2(L - \ell) + 4}  L.
		\end{equation}
		Note that for $\ell > L + 2$, we have $\abs{E_\ell^A} < \frac{2L}{\sqrt{\eta}} < 1$ (i.e., $E_\ell^A = \emptyset$), which implies
		\begin{equation}
			\label{eq:l-3}
			\forall \ell > L + 2: \Pr_{I \sim \nu, \W^A}\left[T \cap E^A_\ell  \neq \emptyset \right] = 0.
		\end{equation}
		We thus only focus on $\ell \le L + 2$ in the RHS of (\ref{eq:l-1}).
		
		For any $\ell = 6(r+1), \dots, L+2$, define the following event
		\begin{equation}
			\label{eq:E-A}
			\E_\ell^A: \exists i \in E_\ell^A\ \text{with} \ \mu_i^A \in S_{\ell+1}^A, \ \text{s.t.}\ \Theta_i^A \not\in G_{\ell+1}^A.
		\end{equation}
		Intuitively, $\E_\ell^A$ stands for the case that there exists an arm in $E_\ell^A$ whose mean is in $S_{\ell+1}^A$ such that the sequence of pull outcomes on the arm is {\em not} good in the first round.  The following claim shows that $\E_\ell^A$ is unlikely to happen.  Its proof will be given shortly.
		
		\begin{claim}
			\label{cla:good-history}
			For any $\ell = 6(\ell+1), \dots, L+2$	, $\Pr_{I \sim \nu, \W^A}[{\E^A_\ell}] \le n^{-7}$.
		\end{claim}
		
		Consider any fixed $\ell \in \{6(\ell+1), \dots, L+2\}$.
		By the law of total probability, we write
		\begin{eqnarray}
			\Pr_{I \sim \nu, \W^A}[E^A_\ell \cap T \neq \emptyset ] &= &\Pr_{I \sim \nu, \W^A}\left[E^A_\ell \cap T \neq \emptyset, \neg \E^A_\ell \right] + \nonumber \\
			&& \Pr_{I \sim \nu, \W^A}[E^A_\ell \cap T \neq \emptyset, \E^A_\ell]. 	\label{eq:n-1}
		\end{eqnarray}
		By Claim~\ref{cla:good-history}, we have
		\begin{equation}
			\label{eq:n-11}
			\Pr_{I \sim \nu, \W^A}[E^A_\ell \cap T \neq \emptyset, \E^A_\ell] \le \Pr_{I \sim \nu, \W^A}[\E^A_\ell] \le n^{-7}.
		\end{equation}
		We further expand the first term in the RHS of (\ref{eq:n-1}):
		\begin{eqnarray}
			&&\Pr_{I \sim \nu, \W^A}\left[E^A_\ell \cap T \neq \emptyset, \neg \E^A_\ell\right] \nonumber \\
			&=& \sum_{h^A} \left(\Pr_{I \sim \nu}[E^A_\ell \cap T \neq \emptyset, \neg \E^A_\ell \ |\ \W^A = h^A]   \Pr_{I \sim \nu, \W^A}[\W^A = h^A] \right) \nonumber \\
			&\le& \sum_{h^A} \left(\sum_{i \in E_\ell^A} \Pr_{I \sim \nu}[\mu_i = \mu_*, \neg \E^A_\ell \ |\ \W^A = h^A]   \Pr_{I \sim \nu, \W^A}[\W^A = h^A] \right). \nonumber \\ 
			&& \label{eq:n-2}
		\end{eqnarray}
		We next bound $\Pr_{I \sim \nu}[\mu_i = \mu_*, \neg \E^A_\ell \ |\ \W^A = h^A]$ for each fixed pull history $h^A$ (or, fixed $\theta_i^A$ for arm $i$).  Since $\mu_i^B$ is independent of $(\mu_i^A, \W^A, \E_\ell^A)$, we have
		\begin{eqnarray}
			&& \Pr_{I \sim \nu}[\mu_i = \mu_*, \neg \E^A_\ell \mid \W^A = h^A] \nonumber \\
			&=& \Pr_{I \sim \nu}[\mu^B_i = \mu^B_*] \cdot \Pr_{I \sim \nu}[\mu^A_i = \mu^A_*, \neg \E^A_\ell \mid \W^A = h^A].
			\label{eq:n-3}
		\end{eqnarray}
		Since $\sigma_i^B \in \D_{6r}^B$, by Lemma~\ref{lem:best-local-arm-2} we have
		\begin{equation}
			\label{eq:n-4}
			\Pr_{I \sim \nu}[\mu^B_i = \mu^B_*] = \Pr_{\mu^B_i \sim \sigma^B_i}[\mu^B_i = \mu^B_*] \le e^{\frac{24 r}{\eta}} \cdot \eta^{-2 d_0},
		\end{equation}
		where $d_0 = \abs{\{k \in \mathbb Z \mid 6r < 2k \le L\}}$.
		
		For the second term in the RHS of (\ref{eq:n-3}), we have
		\begin{eqnarray}
			&&\Pr_{I \sim \nu}\left[\mu_i^A = \mu^A_*, \neg \E^A_\ell \mid \W^A = h^A\right] \nonumber \\ 
			&=& \Pr_{I \sim \nu}\left[ \mu^A_i = \mu^A_*, \neg \E^A_\ell \mid \mu_i^A \in S^A_{\ell + 1}, \W^A = h^A\right] \nonumber \\
			&& \times \Pr_{I \sim \nu}\left[\mu^A_i \in S^A_{\ell + 1} \mid \W^A = h^A\right] + \nonumber \\ 
			&&\  \Pr_{I \sim \nu}\left[\mu^A_i = \mu^A_*, \neg \E^A_\ell \mid \mu_i^A \not\in S^A_{\ell + 1}, \W^A = h^A\right] \nonumber \\
			&& \times \Pr_{I \sim \nu}\left[\mu^A_i \not\in S^A_{\ell + 1} \mid \W^A = h^A\right] \label{eq:k-2} \\
			&\le& \Pr_{I \sim \nu}\left[ \mu^A_i = \mu^A_*, \neg \E^A_\ell \mid \mu_i^A \in S^A_{\ell + 1}, \W^A = h^A\right] \cdot 1  + 0   \label{eq:k-3} \\
			&\le& \Pr_{I \sim \nu}\left[ \mu^A_i = \mu^A_*, \Theta_i^A \in G_{\ell+1}^A \mid \mu_i^A \in S^A_{\ell + 1}, \W^A = h^A\right]  \label{eq:k-4} \\
			&=& \Pr_{\mu_i^A \sim \sigma_i^A}\left[ \mu^A_i = \mu^A_*, \Theta_i^A \in G_{\ell+1}^A \mid \mu_i^A \in S^A_{\ell + 1}, \Theta_i^A = \theta_i^A\right]  \label{eq:k-5} \\
			&\le& e^{\frac{4(\ell + 1)}{\eta}} \cdot \eta^{-2 d_1},  \label{eq:k-6}
		\end{eqnarray}
		where $d_1 = \abs{\{k \in \mathbb Z \mid \ell + 1 < 2k  + 1 \le L\}}$.
		From (\ref{eq:k-2}) to (\ref{eq:k-3}) we have used the fact that if $\mu_i^A = \mu_*^A$, it is impossible that $\mu_i^A \not\in S_\ell^A$. From (\ref{eq:k-3}) to (\ref{eq:k-4}), we have used the definition of $\E_\ell^A$ in (\ref{eq:E-A}).  From (\ref{eq:k-4}) to (\ref{eq:k-5}), we have used the independence of $\mu_1^A, \ldots, \mu_\kappa^A$.  From (\ref{eq:k-5}) to (\ref{eq:k-6}), if $\theta_i^A \not\in G_{\ell+1}^A$, then (\ref{eq:k-5}) is $0$; otherwise if $\theta_i^A \in G_{\ell+1}^A$, then by  Lemma~\ref{lem:distribution-class} we know that $(\mu_i^A\ |\ \mu_i^A \in S_{\ell+1}^A, \text{good } \theta_i^A) \in \D_{\ell+1}^A$, and the inequality follows by Lemma~\ref{lem:best-local-arm}.
		
		Plugging (\ref{eq:n-4}) and (\ref{eq:k-6}) to (\ref{eq:n-3}), we have
		\begin{eqnarray}
			\Pr_{I \sim \nu}[\mu_i = \mu_*, \neg \E^A_\ell \mid \W^A = h^A] &\le& e^{\frac{4(\ell+1+6r)}{\eta}} \cdot \eta^{-2(d_0+d_1)} \nonumber \\ &\le& e^{\frac{8(L+2)}{\eta}}  \cdot \eta^{-2(d_0+d_1)},
			\label{eq:o-1}
		\end{eqnarray}
		where in the last inequality we have used the fact that $\ell \ge 6(r+1)$ (by the definition of $E^A$ in (\ref{eq:set-E})) and our focus $\ell \le L+2$.  
		
		For $d_0$ and  $d_1$, we have
		\begin{eqnarray}
			&&d_0 + d_1  \nonumber \\
			&=& \abs{\{k \in \mathbb Z \mid 6r < 2 k \le L\}} + \abs{\{k \in \mathbb Z \mid \ell + 1 < 2 k  + 1\le L\}}  \nonumber
			\\ &=& \abs{\{k \in \mathbb Z \mid 6r < 2 k \le \ell + 1\}} + \abs{\{k \in \mathbb Z \mid \ell + 1 < k \le L\}} \nonumber \\
			& \ge & 3 + L - \ell. \label{eq:o-2}
		\end{eqnarray}
		Plugging (\ref{eq:o-2}) back to (\ref{eq:o-1}), we have
		\begin{equation}
			\label{eq:o-3}
			\Pr_{I \sim \nu}[\mu_i = \mu_*, \neg \E^A_\ell \mid \W^A = h^A] \le e^{\frac{8(L+2)}{\eta}} \cdot \eta^{-2(3+L-\ell)}.
		\end{equation}
		
		Plugging (\ref{eq:o-3}) to (\ref{eq:n-2}) and using (\ref{eq:l-2}), we have 
		\begin{eqnarray}
			&&\Pr_{I \sim \nu, \W^A}\left[E^A_\ell \cap T \neq \emptyset, \neg \E^A_\ell\right] \nonumber \\ &=& \sum_{h^A} \left(\sum_{i \in E_\ell^A} \Pr_{I \sim \nu}[\mu_i = \mu_*, \neg \E^A_\ell \ |\ \W^A = h^A] \cdot \Pr_{I \sim \nu, \W^A}[\W^A = h^A] \right)  \nonumber \\
			&\le&  \left(2 \eta^{-\frac{1}{2}}\eta^{2(L - \ell) + 4}  L\right) \cdot \left(e^{\frac{8(L+2)}{\eta}} \eta^{-2(3+L-\ell)}\right)  \nonumber \\
			&\le& 2e^{\frac{8(L+2)}{\eta}} L \eta^{-\frac{5}{2}}. \label{eq:o-6}
		\end{eqnarray}
		
		Finally, we have
		\begin{eqnarray*}
			&&\Pr_{I \sim \nu, \W^A}\left[T \cap E^A \neq \emptyset\right] \nonumber \\ &\stackrel{(\ref{eq:l-1})}{=}& \sum_{\ell =6 (r + 1)}^{\infty} \Pr_{I \sim \nu, \W^A}\left[T \cap E^A_\ell  \neq \emptyset \right] \\
			&\stackrel{(\ref{eq:l-3})}{=}& \sum_{\ell =6 (r + 1)}^{L+2} \Pr_{I \sim \nu, \W^A}\left[T \cap E^A_\ell  \neq \emptyset \right] \\
			&\stackrel{(\ref{eq:n-1})}{=}& \sum_{\ell =6 (r + 1)}^{L+2} \left(\Pr_{I \sim \nu, \W^A}\left[E^A_\ell \cap T \neq \emptyset, \neg \E^A_\ell \right] \right.\\
			&&\left. + \Pr_{I \sim \nu, \W^A}[E^A_\ell \cap T \neq \emptyset, \E^A_\ell] \right)\\
			&\stackrel{(\ref{eq:n-11}), (\ref{eq:o-6})}{=}& (L+2) \cdot \left(2e^{\frac{8(L+2)}{\eta}} L \eta^{-\frac{5}{2}} + n^{-7}\right) \\
			&\le& 2e^{\frac{10L}{\eta}} L^2 \eta^{-\frac{5}{2}} + n^{-6}.
		\end{eqnarray*}
	\end{proof}
	
	\begin{proof}[Proof of Claim~\ref{cla:good-history}]
		Fix any $\ell \in \{6(\ell+1), \ldots, L+2\}$. For each $i \in E_\ell^A$, we have
		\begin{eqnarray}
			&&\Pr_{\mu_i^A \sim \sigma_i^A, \Theta_i^A} \left[\Theta_i^A \not\in G_{\ell+1}^A \ \left|\  \mu_i^A \in S_{\ell+1}^A \right. \right] \label{eq:p-1} \\
			&\le& \sum_{q = \gamma \eta^{2(\ell-1)}+1}^{\gamma \eta^{2\ell}} \left(\Pr_{\mu_i^A \sim \sigma_i^A, \Theta_i^A \sim \Ber(\mu_i^A)^{\otimes q}} \left[\left. \Theta_i^A \not\in G_{\ell+1}^A \ \right|\ \mu_i^A \in S_{\ell+1}^A \right] \right) \nonumber \\
			&& \label{eq:p-2} \\
			&\le&  \gamma \eta^{2\ell} \cdot n^{-10}   \label{eq:p-3} \\
			&\le& n^{-8},  \label{eq:p-4}
		\end{eqnarray}
		where in (\ref{eq:p-1}), $\Theta_i^A$ is distributed according to $\Ber(\mu_i^A)^{\otimes Q}$ for a random variable $Q \in (\gamma \eta^{2(\ell-1)}, \gamma \eta^{2\ell}]$ (determined by $\W^A$).
		From (\ref{eq:p-1}) to (\ref{eq:p-2}), we have used a union bound on all possible values of $Q = q$.
		From (\ref{eq:p-2}) to (\ref{eq:p-3}), we have used Lemma~\ref{lem:good-pull-outcome}; note that $q \in (\gamma \eta^{2(\ell-1)}, \gamma \eta^{2\ell}] \subseteq [\eta^3, \frac{\eta^{2(\ell+1)-1}}{2^7}]$, satisfying the requirement of Lemma~\ref{lem:good-pull-outcome}.
		
		By (\ref{eq:p-4}) and the definition of event $\E_\ell^A$ in (\ref{eq:E-A}) we have 
		\begin{eqnarray*}
			\Pr_{I \sim \nu, \W^A}[{\E_\ell^A}] &\le& \sum_{i \in E_\ell^A} \Pr_{\mu_i^A \sim \sigma_i^A, \Theta_i^A} \left[\Theta_i^A \not\in G_{\ell+1}^A \ \left|\ \mu_i^A \in S_{\ell+1}^A \right. \right] \\
			&\le& n \cdot n^{-8} = n^{-7}.
		\end{eqnarray*}
	\end{proof}
\else
\fi

The following lemma upper bounds the second probability. It is symmetric to Lemma~\ref{lem:publish-alice}, and can be proved using a similar line of arguments.  

\begin{lemmaprime}{lem:publish-alice}
	\label{lem:publish-bob}
		$	\Pr_{I \sim \nu, \W^A, \W^B}\left[T \cap E^B \neq \emptyset\right] \le  2e^{\frac{10L}{\eta}} L^2 \eta^{-\frac{5}{2}} + n^{-6}. $
\end{lemmaprime}

The next lemma upper bounds the third probability.
\begin{lemma}
	\label{lem:third-term}
	Let $\tilde{\A}$ be the $(R-(r+1))$-round algorithm obtained from $\A$, conditioned on the pull history of the first round being $\W^A$ and $\W^B$. We have
	\begin{equation*}
		\Pr_{{I \sim \nu, \W^A, \W^B}}\left[\tilde{A} \text{ succeeds on } \left(P^A \cap P^B\right) \backslash E\right] \le \lambda_{r + 1} + 2  n^{-9}.
	\end{equation*}
\end{lemma}

Before proving Lemma~\ref{lem:third-term}, we begin with some preparation.
Define two events 
\begin{eqnarray}
	\chi^A : \exists i \in P^A \backslash E\quad \text{s.t.} \quad \Theta^A_i \not\in G^A_{6(r + 1)}, \label{eq:q-1} \\
	\chi^B : \exists i \in P^B \backslash E\quad \text{s.t.}\quad \Theta^B_i \not\in G^B_{6(r + 1)}.	 \label{eq:q-2}
\end{eqnarray}
In the next two lemmas, we show that $\chi^A$ and $\chi^B$ do {\em not} happen with high probability.

\begin{lemma}
	\label{lem:chi-alice}
		$\Pr_{I \sim \nu, \W^A, \W^B}[\chi^A] \le n^{-9}.$
\end{lemma} 

\begin{proof}
	Recall that each arm in $P^A \backslash E$ has been pulled for $q = \gamma \eta^{2(6(r+1)-1)} \in [\eta^3, \frac{\eta^{2(6(r+1))-1}}{2^7}]$ times. 
	\begin{eqnarray}
		&&\Pr_{I \sim \nu, \W^A, \W^B}[\chi^A] \nonumber \\ &\le&  \sum_{i=1}^\kappa   \Pr_{\mu^A_i \sim \sigma_i^A, \Theta_i^A  \sim \Ber(\mu_i^A)^{\otimes q}}\left[ \Theta_i^A \not\in G^A_{6(r + 1)}\ \left|\ \mu_i^A \in S^A_{6(r + 1)} \right.\right] \nonumber \\ &&\label{eq:r-2} \\
		&\le& n \cdot n^{-10}  = n^{-9}, \label{eq:r-3}
	\end{eqnarray}
	where from (\ref{eq:r-2}) to (\ref{eq:r-3}) we have used Lemma~\ref{lem:good-pull-outcome}.
\end{proof}

The following lemma is symmetric to Lemma~\ref{lem:chi-alice}, and can be proved using a similar line of arguments.
\begin{lemmaprime}{lem:chi-alice}
	\label{lem:chi-bob}
		$		\Pr_{I \sim \nu, \W^A, \W^B}[\chi^B] \le n^{-9}. $
\end{lemmaprime} 

\begin{proof}[Proof of Lemma~\ref{lem:third-term}]
	For the convenience of writing, we further introduce the following event.
	\begin{equation}
		\psi: \tilde{\A}\ \text{succeeds on}\ \left(P^A \cap P^B\right) \backslash E.
	\end{equation}
	We write
	\begin{eqnarray}
		&&\Pr_{I \sim \nu, \W^A, \W^B}[\psi] \nonumber \\
		& \le & \Pr_{I \sim \nu, \W^A, \W^B}[\psi, \neg \chi^A, \neg \chi^B] + \Pr_{I \sim \nu, \W^A, \W^B}[\chi^A] \nonumber \\ 
		&& + \Pr_{I \sim \nu, \W^A, \W^B}[\chi^B]  \label{eq:s-1} \\
		& \le & \Pr_{I \sim \nu, \W^A, \W^B}[\psi, \neg \chi^A, \neg \chi^B] + 2 n^{-9} \label{eq:s-2} \\
		& = & \sum_{(h^A, h^B)} \left( \Pr_{I \sim \nu}[\psi, \neg \chi^A, \neg \chi^B\ |\ (\W^A, \W^B) = (h^A, h^B)] \right. \nonumber \\
		&& \left.  \Pr_{\W^A, \W^B} [(\W^A, \W^B) = (h^A, h^B)] \right) + 2 n^{-9}, \label{eq:s-3}
	\end{eqnarray}
	where from (\ref{eq:s-1}) to (\ref{eq:s-2}) we have used Lemma~\ref{lem:chi-alice} and Lemma~\ref{lem:chi-bob}.
	
	Consider a fixed pull history $(h^A, h^B)$. For any $i \in (P^A \cap P^B) \backslash E$, its sequence of pull outcomes $(\theta_i^A, \theta_i^B)$ in the first round is fully determined by $(h^A, h^B)$.  We consider two cases.
	\medskip
	
	\noindent{\em Case I: $\chi^A$ {\em or} $\chi^B$ holds.}\ \  In this case, we have 
	\begin{equation}
		\label{eq:s-4}
		\Pr_{I \sim \nu}[\psi, \neg \chi^A, \neg \chi^B\ |\ (\W^A, \W^B) = (h^A, h^B)] = 0.
	\end{equation}

	\noindent{\em Case II: $\neg \chi^A$ {\em and} $\neg \chi^B$ holds.} \ \ 
	In this case, by the definition of $\chi^A$ in (\ref{eq:q-1}) and $\chi^B$ in  (\ref{eq:q-2}), we have for any $i \in (P^A \cap P^B) \backslash E$, $\theta_i^A \in G_{6(r+1)}^A$ and $\theta_i^B \in G_{6(r+1)}^B$.
	The posterior distribution of the local mean of arm $i$ at Alice's side can be written as
	\begin{eqnarray*}
		\tilde{\sigma}_i^A  &=& \left(\sigma_i^A \ \left|\ \mu_i^A \in S^A_{6(r + 1)}, (\W^A, \W^B) = (h^A, h^B) \right.\right) \\ &=& \left(\sigma_i^A \mid \mu_i^A \in S_{6(r + 1)}^A, \Theta_i^A = \theta_i^A \in G_{6(r+1)}^A \right) \in \D^A_{6(r + 1)}. \label{eq:t-1}
	\end{eqnarray*}
	Similarly, the posterior distribution of the local mean of arm $i$ at Bob's side can be written as
	\begin{equation*}
		\tilde{\sigma}_i^B = \left(\sigma_i^B \ \left|\ \mu_i^B \in S^B_{6(r + 1)},  (\W^A, \W^B) = (h^A, h^B) \right.\right) \in \D_{6(r + 1)}^B. \label{eq:t-2}
	\end{equation*}
	Thus, for any $i \in (P^A \cap P^B) \backslash E$,  we have $\tilde{\sigma}_i = (\tilde{\sigma}_i^A, \tilde{\sigma}_i^B) \in \D_{6(r+1)}$.  
	Recall that $\tilde{\A}$ is a $(R - (r+1))$-round algorithm working on a set of arms $(P^A \cap P^B) \backslash E$ with $\tilde{\kappa} = \abs{(P^A \cap P^B) \backslash E} \le n$, conditioned on the first round pull history being $(\W^A, \W^B)$. By the definition of $\lambda_{r+1}$ in (\ref{eq:j-6}) and the fact that conditioned on the pull history $(\W^A, \W^B)$, the distribution of the $\tilde{\kappa}$ arms belongs to ${\tilde{\sigma}}^{\otimes \tilde{\kappa}} \in \D_{6(r+1)}^{\otimes \tilde{\kappa}}$,
	\begin{equation}
		\label{eq:u-1}
		\Pr_{I \sim \nu}[\psi, \neg \chi^A, \neg \chi^B\ |\ (\W^A, \W^B) = (h^A, h^B)]  \le \lambda_{r+1}.
	\end{equation} 
	\smallskip
	
	Combining (\ref{eq:s-3}), (\ref{eq:s-4}) and (\ref{eq:u-1}), we have
	\begin{eqnarray*}
		\label{eq:u-2}
		&& \Pr_{I \sim \nu, \W^A, \W^B}[\psi]  \nonumber \\
		&\le& \sum_{(h^A, h^B)} \left( \lambda_{r+1}
		\Pr_{\W^A, \W^B} [(\W^A, \W^B) = (h^A, h^B)] \right) + 2 n^{-9} \nonumber \\
		&\le&  \lambda_{r+1} + 2n^{-9}.
	\end{eqnarray*} 
\end{proof}

\noindent{\bf Summing Up.\ \ }
Combining Lemma~\ref{lem:publish-alice},  Lemma~\ref{lem:publish-bob}, and Lemma~\ref{lem:third-term}, 
\begin{eqnarray}
	&&\Pr_{I \sim \nu, \W^A, \W^B}[\A\ \text{succeeds on}\ I] \nonumber \\
	&\le&  \Pr_{I \sim \nu, \W^A, \W^B}[T \cap E^A \neq \emptyset] + \Pr_{I \sim \nu, \W^A, \W^B}[T \cap E^B \neq \emptyset] \nonumber  \\
	&& \quad + \Pr_{I \sim \nu, \W^A, \W^B}\left[\tilde{A} \text{ succeeds on } \left(P^A \cap P^B\right) \backslash E\right] \label{eq:v-1} \\
	&\le&  4\left(e^{\frac{10L}{\eta}} L^2 \eta^{-\frac{5}{2}} + n^{-6}\right) + (\lambda_{r+1} + 2n^{-9}) \label{eq:v-2} \\
	&\le& \lambda_{r+1} + 4e^{\frac{10L}{\eta}} L^2 \eta^{-\frac{5}{2}} + n^{-5}.  \label{eq:v-3}
\end{eqnarray}
Since (\ref{eq:v-3}) holds for any algorithm $\A$, distribution $\nu$, and $\kappa \in [n]$, we have $\lambda_r \le \lambda_{r+1} + 4e^{\frac{10L}{\eta}} L^2 \eta^{-\frac{5}{2}} + n^{-5}$.

\vspace{2mm}
\noindent{\bf The Base Case.\ \ }
In the base case we consider $0$-round algorithm (i.e., when $r = R$).  We have the following lemma.
\begin{lemma}
	\label{lem:base-case}
	For $R = \frac{L}{6}$, $\lambda_R \le e^{\frac{48R}{\eta}} \eta^{-2}$.
\end{lemma}

\begin{proof}
	Any $0$-round algorithm needs to output an arm $i$ as the best arm without making any pulls.  For any $i$ with mean $\mu_i \sim \sigma_i \in \D_{6R}$, by Lemma~\ref{lem:best-local-arm} and Lemma~\ref{lem:best-local-arm-2}, we have
	\begin{eqnarray}
		\label{eq:w-1}
		\Pr_{\mu_i \sim \sigma_i}[\mu_i = \mu_*] &=& \Pr_{\mu_i^A \sim \sigma_i^A}[\mu_i^A = \mu_*^A] \Pr_{\mu_i^B \sim \sigma_i^B}[\mu_i^B = \mu_*^B] \\
		&\le& e^{\frac{4\cdot 6R}{\eta}} \eta^{-2d_1} \cdot e^{\frac{4\cdot 6R}{\eta}} \eta^{-2d_0} \\
		&=& e^{\frac{48R}{\eta}} \eta^{-2(d_0+d_1)},
	\end{eqnarray}
	where $d_0 + d_1 = \abs{\{k\ |\ 6R < k \le L\}}$.  For $R = \frac{L}{6}$, we have
	$\Pr_{\mu_i \sim \sigma_i}[\mu_i = \mu_*] \le e^{\frac{48R}{\eta}} \eta^{-2}.$
\end{proof}

\vspace{2mm}
\noindent{\bf Putting Things Together (Proof for Theorem~\ref{thm:lb-two}).\ \ }
By Lemma~\ref{lem:base-case} and Lemma~\ref{lem:induction}, we have
\begin{eqnarray*}
	\label{eq:x-1}
	\lambda_0 &\le& \lambda_R + R \cdot \left(4e^{\frac{10L}{\eta}} L^2 \eta^{-\frac{5}{2}} + n^{-5}\right) \\
	&\le& e^{\frac{48R}{\eta}} \eta^{-2} + (L/6)\cdot \left(4e^{\frac{10L}{\eta}} L^2 \eta^{-\frac{5}{2}} + n^{-5}\right) \le \eta^{-1}.
\end{eqnarray*}
Therefore, any $R$-round collaborative algorithm that uses $2\zeta \eta^{2+2L} L$ time (i.e., each agent can make at most $2\zeta \eta^{2+2L} L$ pulls in total) can succeed with probability at most $\eta^{-1}$.

Recall the definition of event $\E_1$ in (\ref{eq:event-E-1}): $\exists$ a unique $i^* \in [n]$ such that $\mu_{i^*} = \mu_*$ {\em and} the instance complexity $H = H(I) \le 2\eta^{2+2L} L$ where $I \sim (\pi^{\otimes n}\ |\ \E_1)$. 

By Lemma~\ref{lem:E-1}, $\Pr[\E_1] \ge 1/(2e)$.  We thus have
\begin{eqnarray*} 
	\label{eq:x-2}
	\Pr_{I \sim (\pi^{\otimes n}|\E_1)}[\A\ \text{succeeds on}\ I] &\le& \frac{\Pr_{I \sim \pi^{\otimes n}}[\A\ \text{succeeds on}\ I]}{\Pr_{I \sim \pi^{\otimes n}}[\E_1]} \\
	&\le& \lambda_0 \cdot (2e) \\
	&\stackrel{(\ref{eq:x-1})}{\le}& \frac{2e}{\eta} < 0.9\ .
\end{eqnarray*}
Therefore, any $R$-round ($1 \le R \le \frac{\log n}{24\log\log n}$) collaborative algorithm that succeeds on input distribution $(\pi^{\otimes n}\ |\ \E_1)$ with probability at least $0.9$ needs time at least 
$2\zeta \eta^{2+2L} L \ge H \cdot \zeta \ge H \cdot n^{\frac{1}{25R}}.$ 

\subsection{General $K$}
\label{sec:general-K}

We now consider the general case where there are $K$ agents.  The following theorem is a restatement of Theorem~\ref{thm:lb-main}.

\begin{theorem}
	\label{thm:lb-general}
	For any $1 \le R \le \frac{\log n}{24\log\log n}$, any $R$-round $K$-agent algorithm that solves $n$-arm BAI in the heterogeneous CL mode with probability $0.99$ uses time at least $H n^{\Omega\left(\frac{1}{R}\right)}/K$.
\end{theorem}

\begin{proof}
	We prove the general $K$ case by a reduction from the $K = 2$ case.  Suppose there exists a $R$-round algorithm for BAI in the heterogeneous CL model with $n$ arms using $K$ agents and uses time smaller than $H n^{\frac{1}{26R}} / K$, we show that there also exists a $R$-round algorithm for the same problem using $2$ agents and uses time smaller than $H n^{\frac{1}{25R}}$, contradicting Theorem~\ref{thm:lb-two}.
	
	The reduction works as follows. Given any algorithm $\A$ for the $K$-agent case, we construct an algorithm $\A'$ for the $2$-agent case: We divide the $K$ agents to two groups each having $K/2$ agents. Let Alice simulate the first group, and Bob simulate the second group. In each round, the sequence of arm pulls Alice makes is simply the {\em concatenation} of arm pulls made by the $K/2$ agents that she simulates, and the sequence of arm pulls Bob makes is the concatenation of arm pulls made by the $K/2$ agents that he simulates. The messages sent by Alice in each communication step is a concatenation of the messages sent by agents in the group she simulates in the corresponding communication step in $\A$; similar for Bob.
	Now if $\A$ uses time at most $H n^{\frac{1}{26R}} / K$, then $\A'$ uses time at most
	$H n^{\frac{1}{26R}}/K \cdot (K/2) < H n^{\frac{1}{25R}},$
	contradicting to Theorem~\ref{thm:lb-two}.
\end{proof}

\section{The Algorithm}\label{sec:algo}

In this section, we present a CL algorithm that gives Theorem~\ref{thm:ub-main}.  Our algorithm is {\em non-adaptive}. It follows the successive elimination approach, and can be seen as a generalization of the algorithm for the heterogeneous CL setting in \cite{KZ23} to the entire time-round tradeoff curve.

Intuitively, we partition the learning process into $R$ rounds with predefined lengths $t_1, \ldots, t_R$.  In each round $r$, each of the $K$ agents simply pulls each remaining arm for $t_r$ times. At the end of each round, the $K$ agents communicate and compute the global empirical means of each arm, and then select the $n_r$ arms with the highest global empirical means and proceed to the next round, where $n_1, \ldots, n_R$ are also predefined.  We set $n_R$ to be $1$ so that at the end of the $R$-round, there will be just one arm left, which can be proven to be the best arm with high probability.

The algorithm is described in Algorithm~\ref{alg:ub}. It gives the following guarantees. 
\ifdefined\fullversion
\else
Due to the space constraints, we leave its proof to the full version of this paper~\cite{KZ24}.
\fi

\begin{algorithm}[t]
	\caption{\textsc{CL-Heterogeneous}$(I, R, T)$}\label{alg:ub}
	\textbf{Input:} a set of $n$ arms $I$, round parameter $R$, number of agents \(K\),
	and time horizon $T$. \\
	\textbf{Output: } the arm with the largest global mean. \\
	Initialize $I_0 = I$; \\
	set $T_0 \gets 0$, $T_r \gets \left\lfloor\frac{n^{r/R} T}{n^{1 + 1/R} R}\right\rfloor$ for $r = 1, \dotsc, R$; \\
	set $n_r \gets \left\lfloor \frac{n}{n^{r/R}}\right\rfloor$ for $r = 0, \dotsc, R - 1$, and $n_R \gets 1$; \\
	\For{$r = 0, 1, \dotsc, R-1$}{
		each agent pulls each arm in $I_r$ for $(T_{r+1} - T_r)$ times; \\
		the $k$-th agent computes the local empirical mean $\hat{\mu}^{(r)}_{i,k}$ for $i \in I_r$; \\
		let $\hat{\mu}^{(r)}_{i} \gets \frac{1}{K}\sum_{k \in [K]} \hat{\mu}^{(r)}_{i, k}$; \\
		let $I_{r + 1}$ be the set of $n_{r+1}$ arms in $I_r$ with the highest global empirical means $\hat{\mu}^{(r)}_{i}$\;
	}
	\Return the single element in $I_R$.
\end{algorithm}

\begin{theorem}\label{thm:ub}
	For any $R \ge 1$, Algorithm~\ref{alg:ub} solves BAI in the heterogeneous CL model with $K$ agents and $n$ arms using $T$ time steps and $R$ rounds, with a success probability at least
	\begin{equation}
		1 - 2nR \cdot \exp\left(-{KT}/{(2 n^{\frac{1}{R}} R H)}\right).
	\end{equation}
\end{theorem}

Note that Theorem~\ref{thm:ub-main} (in the introduction) is an immediate corollary of Theorem~\ref{thm:ub}.

\ifdefined\fullversion
	In the rest of this section, we prove Theorem~\ref{thm:ub}.
	
	It is clear that Algorithm~\ref{alg:ub} uses $R$ rounds. The learning time of the algorithm can be bounded by
	\begin{equation}
		\sum_{r=0}^{R-1} n_r (T_{r+1} - T_r) \le \sum_{r=0}^{R-1} n_r T_{r+1} \le \sum_{r = 0}^{R-1} \frac{n}{n^{r/R}} \cdot \frac{n^{r/R} T}{n R} \le T.
	\end{equation}
	
	We next bound the error probability of the algorithm. Let $\pi : [n] \to [n]$ be the bijection such that $\mu_{\pi(1)} \ge \dotsc \ge \mu_{\pi(n)}$. For $i \neq \pi(1)$, let
	\begin{equation}
		\Delta_i \triangleq \mu_{\pi(1)} - \mu_{i}\,,
	\end{equation}
	and for $i = \pi(1)$, let
	\begin{equation}
		\Delta_i \triangleq \mu_i - \mu_{\pi(2)}\,.
	\end{equation}
	We define the following event which we will condition on in the rest of the proof.
	\begin{equation}
		\E_a : \left\{\forall{r \in \{0, \dotsc, R-1\}, \forall{i \in I_r}}, \abs{\hat{\mu}^{(r)}_i - \mu_i} < \frac{\Delta_{\pi(n_{r + 1})}}{2}\right\}\,.
	\end{equation}
	The next lemma shows that $\E_a$ happens with a high probability when $T$ is large enough.  
	
	\begin{lemma}
		\begin{equation*}
			\Pr[\E_a] \ge 1 - 2n R \cdot \exp\left(\frac{-KT}{2 n^{1/R} R}\right)\,.
		\end{equation*}
	\end{lemma}
	
	\begin{proof}
		By a union bound, we can write
		\begin{eqnarray}
			\label{eq:ac-1}
			\Pr[\neg \E_a] & = & \Pr\left[\exists r \in \{0, \dotsc, R-1\}, \exists i \in I_r, \abs{\hat{\mu}^{(r)}_i - \mu_i} \ge \frac{\Delta_{\pi(n_{r + 1})}}{2}\right] \nonumber \\
			& \le & \sum_{r = 0}^{R - 1} \sum_{i = 1}^{n} \Pr\left[\abs{\hat{\mu}^{(r)}_i - \mu_i} \ge \frac{\Delta_{\pi(n_{r + 1})}}{2}\right]
		\end{eqnarray}
		
		Recall that
		\begin{equation*}
			\hat{\mu}^{(r)}_i = \frac{1}{T_{r + 1}} \sum_{t = 1}^{T_{r + 1}}\sum_{k = 1}^{K} \frac{1}{K} X_{k, i, t},
		\end{equation*}
		where $X_{k, i, t}$ are independent random variables with support $[0, 1]$ and mean $\mu_{i, k}$.
		We have
		\begin{eqnarray}
			\label{eq:ac-2}
			\bE\left[\hat{\mu}^{(r)}_i\right] = \frac{1}{T_{r + 1}} \sum_{t = 1}^{T_{r + 1}}\sum_{k = 1}^{K} \frac{1}{K} \bE[ X_{k, i, t}] = \frac{1}{T_{r + 1}} \sum_{t = 1}^{T_{r + 1}}\sum_{k = 1}^{K} \frac{1}{K} \mu_{i, k} = \mu_i\,.
		\end{eqnarray}
		Using Chernoff-Hoeffding bound (Lemma~\ref{lem:chernoff}) for random variables $(\frac{1}{K}X_{i, k, t})$, we have
		\begin{equation}
			\label{eq:ac-3}
			\Pr\left[\abs{\sum_{t = 1}^{T_{r + 1}} \sum_{k = 1}^K \frac{1}{K}  X_{i, k, t} - \mu_i T_{r + 1}} \ge t\right] \le 2 \exp\left(\frac{-2t^2 K}{ T_{r + 1}}\right)\,.
		\end{equation}
		Setting $t = \frac{\Delta_{\pi(n_{r + 1})} T_{r + 1}}{2}$ and $T_{r + 1} = \frac{n^{r / R}  T}{n R}$ in (\ref{eq:ac-3}), we have
		\begin{eqnarray}
			\Pr\left[\abs{\hat{\mu}^{(r)}_{i} - \mu_i} \ge \frac{\Delta_{\pi(n_{r + 1})}}{2} \right]  & \le & 2 \exp\left(\frac{-2 \Delta^2_{\pi(n_{r + 1})} T_{r + 1} K}{4}\right) \nonumber \\
			& \le & 2 \exp\left(\frac{-\Delta^2_{\pi(n_{r + 1})} n^{r / R} T K}{2 n R }\right) \label{eq:ad-1} \\
			& \le & 2 \exp \left(\frac{-n^{r / R} T K}{2 n R } \cdot \frac{n_{r + 1}}{H}\right) \label{eq:ad-2} \\
			& \le & 2 \exp\left(\frac{-n^{r / R} T K}{2 n R } \cdot \frac{n}{n^{r/R}n^{1/R}H}\right) \nonumber\\
			& \le & 2 \exp\left(\frac{-T K}{2 H n^{1/R} R }\right). \label{eq:ad-3}
		\end{eqnarray}
		where from (\ref{eq:ad-1}) to (\ref{eq:ad-2}) we have used
		\begin{equation*}
			\forall{i \in [n]} : H \ge \frac{i}{\Delta^2_{\pi(i)}} .
		\end{equation*}
		
		Plugging~\eqref{eq:ad-3} to~\eqref{eq:ac-1}, we have
		\begin{eqnarray*}
			\Pr[\neg \E_a] & \le & \sum_{r = 0}^{R - 1} \sum_{i = 1}^{n} \Pr\left[\abs{\hat{\mu}^{(r)}_i - \mu_i} \ge \frac{\Delta_{\pi(n_{r + 1})}}{2}\right] \nonumber \\
			& \le & 2 n R \exp\left(\frac{-TK}{2 H n^{1/R} R }\right).
		\end{eqnarray*}
	\end{proof}

	We next show that if $\E_a$ holds, then Algorithm~\ref{alg:ub} outputs the best arm.  Let $i_*$ be the arm with the highest mean. We prove this by showing that for any $r = 0, 1, \ldots, R$, $i_* \in I_r$ ($I_r$ is the set of arms ``survive" in the $r$-th round).  For $r = 0$, the statement holds trivially.   In the $r$-th iteration of Algorithm~\ref{alg:ub}, each agent pulls all arms in $I_r$ for $T_{r+1}$ times.   Consider the subset of arms $Q \subseteq I_r$ of size $(n_{r} - n_{r + 1})$ with the lowest means. For each $i \in Q$, we have $\mu_{i} \le \mu_{\pi(n_{r + 1})}$. Consequently,
	\begin{equation}
		\label{eq:ae-1}
		\forall{i \in Q} : \mu_{i_*} \ge \mu_i + \Delta_{\pi(n_{r + 1})}.
	\end{equation}
	
	By $\E_a$ and~\eqref{eq:ae-1}, we have
	
	\begin{equation*}
		\label{eq:ae-2}
		\forall{i \in Q}: \hat{\mu}_{i_*}^{(r)} > \mu_i + \Delta_{\pi(n_{r + 1})} - \frac{\Delta_{\pi(n_{r + 1})}}{2}  > \hat{\mu}^{(r)}_i,
	\end{equation*}
	which implies that $i_*$ belongs to the set of $n_{r + 1} (= n_r - \abs{Q})$ with the highest means; in other words, $i_* \in I_{r + 1}$. Consequently, we have that if $\E_a$ holds, then the algorithm outputs the correct answer.
\else
\fi

\begin{remark}
	\label{rem:ub-cc}
	We note that the total messages exchanged between the agents in Algorithm~\ref{alg:ub} is $O(nK)$ words, which is optimal (up to a logarithmic factor) based on a lower bound result in \cite{KZ23}.
\end{remark}

\balance

\end{document}